\documentclass{article}

\usepackage[preprint]{arxiv}

\usepackage[utf8]{inputenc} % allow utf-8 input
\usepackage[T1]{fontenc}    % use 8-bit T1 fonts
\usepackage{hyperref}       % hyperlinks
\usepackage{url}            % simple URL typesetting
\usepackage{booktabs}       % professional-quality tables
\usepackage{amsfonts}       % blackboard math symbols
\usepackage{nicefrac}       % compact symbols for 1/2, etc.
\usepackage{microtype}      % microtypography
\usepackage{xcolor}         % colors
\usepackage{graphicx}
\usepackage{multirow}
\usepackage{subcaption}
\usepackage{ulem}
\useunder{\uline}{\ul}{}
\usepackage{stfloats}
\usepackage{amsmath, amsthm, amssymb}
\newtheorem{theorem}{Theorem}

\usepackage{wrapfig}

\title{ForecastGrapher: Redefining Multivariate Time Series Forecasting with Graph Neural Networks}

\author{
    Wanlin Cai\textsuperscript{\rm 1},
    Kun Wang\textsuperscript{\rm 2},
    Hao Wu\textsuperscript{\rm 2},
    Xiaoxu Chen\textsuperscript{\rm 3},
    Yuankai Wu\textsuperscript{\rm 1 \thanks{Corresponding author}}
    \\
    \textsuperscript{\rm 1}Sichuan University \\
    \textsuperscript{\rm 2}University of Science and Technology of China\\
    \textsuperscript{\rm 3}McGill University\\
    caiwanlin@stu.scu.edu.cn, \{wk520529,wuhao2022\}@mail.ustc.edu.cn,\\
    xiaoxu.chen@mail.mcgill.ca, wuyk0@scu.edu.cn
}

\begin{document}

\maketitle

\begin{abstract}
The challenge of effectively learning inter-series correlations for multivariate time series forecasting remains a substantial and unresolved problem. Traditional deep learning models, which are largely dependent on the Transformer paradigm for modeling long sequences, often fail to integrate information from multiple time series into a coherent and universally applicable model. To bridge this gap, our paper presents ForecastGrapher, a framework reconceptualizes multivariate time series forecasting as a node regression task, providing a unique avenue for capturing the intricate temporal dynamics and inter-series correlations. Our approach is underpinned by three pivotal steps: firstly, generating custom node embeddings to reflect the temporal variations within each series; secondly, constructing an adaptive adjacency matrix to encode the inter-series correlations; and thirdly, augmenting the GNNs' expressive power by diversifying the node feature distribution. To enhance this expressive power, we introduce the Group Feature Convolution GNN (GFC-GNN). This model employs a learnable scaler to segment node features into multiple groups and applies one-dimensional convolutions with different kernel lengths to each group prior to the aggregation phase. Consequently, the GFC-GNN method enriches the diversity of node feature distribution in a fully end-to-end fashion. Through extensive experiments and ablation studies, we show that ForecastGrapher surpasses strong baselines and leading published techniques in the domain of multivariate time series forecasting. %Code is available at this repository: \url{https://anonymous.4open.science/r/ForecastGrapher}. 在附录里上传
\end{abstract}

\section{Introduction}

Multivariate time series forecasting is a critical component in predictive analytics, aiming to predict future values of interconnected time series based on their historical trends. Over the past decade, this intricate problem has been intensively tackled using various statistical and machine learning methods~\cite{hyndman2018forecasting}. Recently, Various deep learning models, including Transformer-based~\cite{zhou2021informer} and non-attention mechanisms like MLP~\cite{zeng2023transformers,yi2023frequency, 10.1145/3580305.3599533} and TCNs~\cite{wu2023timesnet}, have been proposed to address the challenges of time series forecasting, demonstrating competitive performance.
%Recently, transformers~\cite{vaswani2017attention} with attention mechanisms, which have garnered considerable success in long sequence modeling, have been introduced to the time series task and attained competitive performance~\cite{zhou2021informer}. Initial research primarily focused on utilizing attention mechanisms within Transformers to capture temporal dependencies in extensive time series data. Nevertheless, recent explorations, such as the work by \cite{zeng2023transformers}, are increasingly scrutinizing the actual efficacy of Transformer-based models in capturing these temporal dependencies. Consequently, several more sophisticated models that do not rely on attention mechanisms, using Multi-Layer Perceptron (MLP) or Temporal Convolutional Networks (TCNs), have been proposed, as noted by \cite{wu2023timesnet}, \cite{10.1145/3580305.3599533}, and \cite{yi2023frequency}, demonstrating competitive performance in comparison to Transformer-based forecasters.

Different structures process time series in various ways. Essentially, they all utilize neural networks to capture both inter-series and intra-series correlations (temporal correlations) in time series data~\cite{cai2023msgnet}. Earlier works often overlooked the inter-series correlation, treating all variables at the same time point as a single token. They employed Transformers~\cite{zhou2021informer, wu2021autoformer, zhou2022fedformer}, MLPs~\cite{zeng2023transformers,das2023long, yi2023frequency}, and TCNs~\cite{woo2021cost,yue2022ts2vec} to capture the temporal correlation across these tokens. However, the importance of inter-series correlation is equally significant. 

Recently, several studies~\cite{zhang2023crossformer,liu2023itransformer,cheng2023rethinking} have been exploring the effectiveness of Transformers in modeling inter-series correlations, rather than focusing solely on temporal correlations. For example, iTransformer~\cite{liu2023itransformer} reconceptualizes individual time series as distinct tokens and employs self-attention to capture inter-series correlations between these tokens. This approach demonstrates that Transformers are more adept at modeling inter-series correlations than temporal correlations.\begin{wrapfigure}[16]{r}{8.5cm}
\centering % Centers the entire figure
  \includegraphics[width=0.6\textwidth]{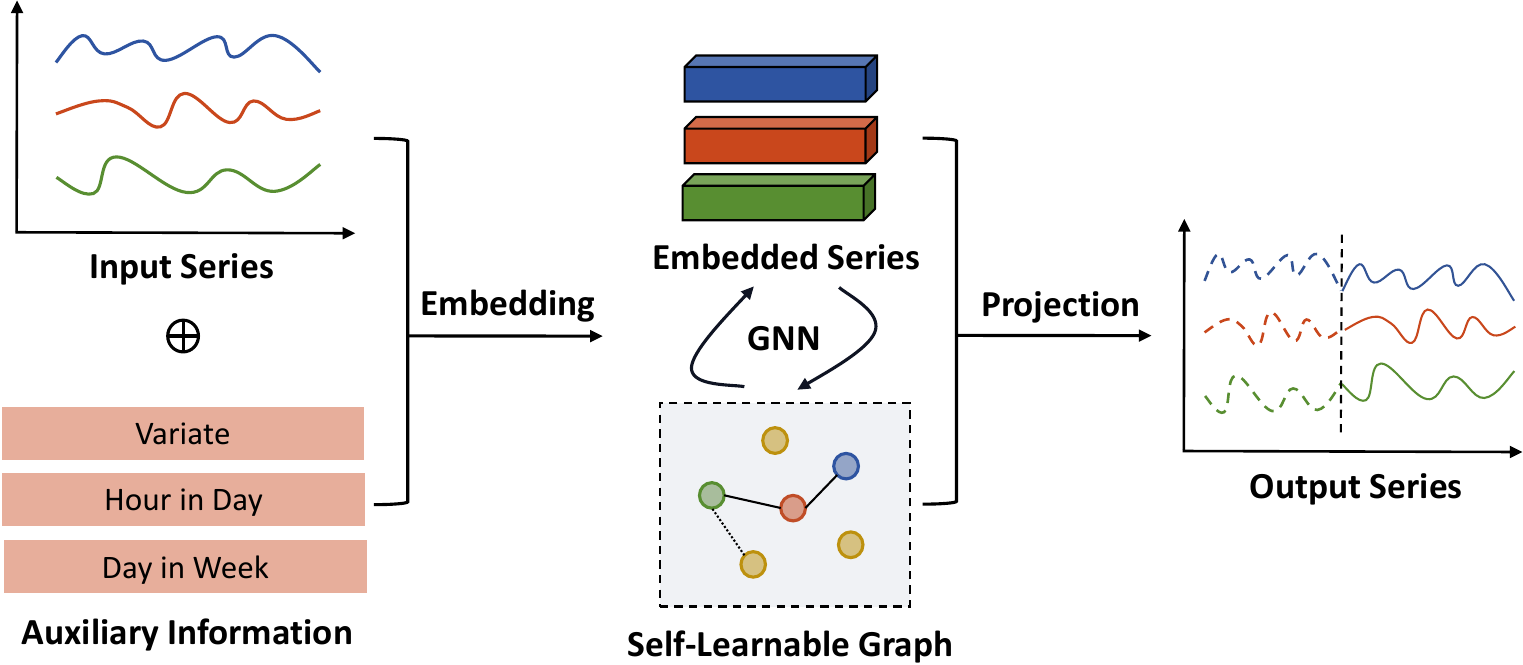} 
  \caption{In ForecastGrapher, each variate is treated as a node within a graph, transforming the multivariate time series forecasting problem into a node regression task.}
  \label{fig:model}
\end{wrapfigure}

The use of Transformers to model inter-series correlations is actually quite similar to graph structure learning in GNNs. This is because the attention mechanisms in Transformers, resembling GNNs' neighborhood aggregation, can be seen as employing dynamic adjacency matrices~\cite{joshi2020transformers,velivckovic2023everything}. This similarity highlights the intriguing parallels and potential for applying GNNs in areas traditionally dominated by Transformers, especially in multivariate time series analysis. At this point, we ask the following important question \-- \textit{Can GNNs models yield superior performance for multivariate time series forecasting, and if so, what adaptations are necessary for multivariate time series data?}

Towards this, we introduce \textbf{ForecastGrapher}, a GNN architecture with strong expressive power tailored for precise multivariate time series forecasting. Figure~\ref{fig:model} illustrates the ForecastGrapher's approach to multivariate time series forecasting. ForecastGrapher conceptualizes each input time series as a graph node. Initially, it employs embedding techniques to encode the temporal variations of individual time series into a high-dimensional space. Subsequently, ForecastGrapher features self-learning graph structures to discern inter-series correlations among nodes. Ultimately, the forecasting results are generated by the form of node regression task after several layers of feature aggregation.

In addition to restructuring the forecasting problem as a node regression task, we also focus on the expressive power issue of GNNs. Notably, although GNNs are widely used in spatio-temporal forecasting (which can be understood as a special type of multivariate time series forecasting problem, mainly focusing on short-term predictions, such as 12-step ahead prediction)~\cite{jin2023spatio}, the inherent limitations in their expressive power remain largely unaddressed in this field. It is well-known that typical GNNs, including Graph Convolutional Networks (GCNs) and Graph Attention Networks (GATs), exhibit limited expressiveness. Several studies~\cite{xu2018powerful,morris2019weisfeiler} have shown that GNNs cannot distinguish between certain pairs of graphs. For instance, %as illustrated in Figure~\ref{fig:gnn_limit}, 
typical GCNs using mean value aggregators are unable to differentiate between node feature distributions with the same mean value but different standard deviations~\cite{corso2020principal, bi2023mm}. This is particularly problematic for multivariate time series forecasting, as differing variances in the historical variables can lead to completely different future outcomes. To address this issue, we introduce a new GNN structure called Group Feature Convolution GNN (GFC-GNN) for ForecastGrapher. We employ learnable scalers to divide features into groups. Within these groups, we perform 1D convolutions across the feature dimension, using kernels of various lengths for differently scaled groups. As a result, GFC-GNN diversifies the distribution of node features in an end-to-end manner.

%\textbf{The key highlights of ForecastGrapher} is its novel approach to treating multivariate time series forecasting as a node regression problem within the graph learning domain. The process involves three crucial steps: 
%1) Generating appropriate node embeddings based on the temporal variation and other external information of individual time series;
%2) Learning an adaptive adjacency matrix for the GNNs;
%3) Enhancing the expressive power of the GNNs. 
%We conducted a comprehensive empirical analysis on twelve widely-used public datasets. The results indicate that ForecastGrapher, a model that effectively addresses these three aspects, consistently surpasses all existing benchmarks across the majority of these datasets. 

To summarize, the key contributions of our ForecastGrapher are outlined as follows:
\begin{itemize}
    \item Framework: We discovered that a GNN architecture designed for the node regression task can effectively address the challenges of multivariate time series forecasting. The key lies in how to learn the graph structure, and in the design of node embeddings and the GNN framework itself.
    \item Expressive Power: Our findings indicate that the application of 1D convolutional layers, with varying kernel lengths to the feature dimensions prior to node aggregation, can effectively diversify the distribution of node features. This enhancement in diversity significantly improves forecasting accuracy.
    \item Performance: ForecastGrapher delivers performance that is comparable to or surpasses state-of-the-art methods across twelve benchmark datasets for multivariate time series forecasting.
\end{itemize}

\section{Related Work}

\subsection{Backbone Networks Used in Time Series Forecasting}

The dominant architecture for time series forecasting has traditionally been the Recurrent Neural Network (RNN)~\cite{salinas2020deepar,rangapuram2018deep, wang2019deep}, the representative work is the introduction of DeepAR~\cite{salinas2020deepar}. Another early representative network for time series forecasting is the Convolutional Neural Network (CNN), with the earliest work being Temporal Convolutional Network (TCN)~\cite{bai2018empirical}, a type of CNN that does not leak future values. CNNs have been successfully utilized in time series forecasting, with notable works including SCInet~\cite{liu2022scinet} and TimesNet~\cite{wu2023timesnet}. Since 2019, the Transformer has been introduced for multivariate time series forecasting, specifically for long sequence modeling~\cite{li2019enhancing,zhou2021informer}.  Subsequently, numerous Transformer variants have been developed to enhance performance in visual tasks. Key advancements include channel dependence~\cite{zhang2023crossformer, liu2023itransformer}, time series patches~\cite{Yuqietal-2023-PatchTST}, and the incorporation of frequency domain information~\cite{zhou2022fedformer,wu2021autoformer}. MLPs have also been explored in the context of time series forecasting~\cite{zeng2023transformers}.
With specially designed modules~\cite{das2023long, yi2023frequency, 10.1145/3580305.3599533, xu2023fits}, MLP can achieve competitive performance. However, most of those methods have not explicitly modeled the inter-series correlation between different variables. Some channel-dependent methods use the attention mechanism to capture relationships between variables, but the correlations obtained often change with time series fluctuations~\cite{zhang2023crossformer, liu2023itransformer}, lacking interpretability.

\subsection{Graph Neural Networks}

The earliest graph neural networks were initially outlined in \cite{gori2005new, scarselli2008graph}. In recent years, a variety of GNN variants have been introduced~\cite{kipf2016semi, atwood2016diffusion, niepert2016learning, gilmer2017neural, velickovic2018graph, abu2019mixhop, corso2020principal}. GNNs are typically applied to data with graph structures, such as social networks~\cite{hamilton2017inductive}, citation networks~\cite{sen2008collective} and biochemical graphs~\cite{wale2008comparison}. Despite their empirical successes in these fields, \cite{xu2018powerful} and
\cite{morris2019weisfeiler} demonstrated that GNNs cannot distinguish some pairs of graphs. To address this limitation, several studies have utilized hand-crafted aggregators to enhance the expressive power of GNNs~\cite{corso2020principal, dehmamy2019understanding, ma2022meta}.

The applications of GNNs in the field of spatio-temporal forecasting involve using GNNs to model spatial attributes, followed by the use of other modules to model temporal attributes. This approach has led to the development of a new Spatio-Temporal GNN (STGNN) structure~\cite{jin2023spatio}. For example, models such as DCRNN~\cite{li2018diffusion}, ST-MetaNet~\cite{pan2019urban}, and AGCRN~\cite{bai2020adaptive} combine GNNs with recurrent neural networks for their operations. Similarly, Graph WaveNet~\cite{wu2019graph}, MTGNN~\cite{wu2020connecting}, and StemGNN~\cite{cao2020spectral} incorporate CNNs for temporal modeling. Additionally, the attention mechanism has become a widely used technique in STGNNs~\cite{guo2019attention,zheng2020gman}. However, these methods only address forecasting problems where both the input and output sequences are short, and they give little consideration to the expressive power of GNNs. Yi et al.~\cite{yi2023fouriergnn} provided a purely GNN-based perspective on multivariate time series forecasting, treating both time points and variates as nodes within a graph. However, their approach is primarily suited to short-term forecasting; for long-term forecasting, treating time points as nodes becomes impractical due to computational complexity.

\section{Methodology}

\subsection{Problem Formulation}

In multivariate time series forecasting, given historical observations $\mathbf{X}_t = \{\mathbf{x}_{t-h}, \ldots, \mathbf{x}_{t-1}\} \in \mathbb{R}^{N \times h}$ with $h$ time steps and $N$ variates, we predict the future $S$ time steps $\mathbf{Y}_t = \{\mathbf{x}_{t}, \ldots, \mathbf{x}_{t+S-1}\} \in \mathbb{R}^{N \times S}$. For convenience, we denote $\mathbf{x}_{t} \in \mathbb{R}^N$ as the time series data collected at time point $t$. Furthermore, we denote $\mathbf{X}_{t,n} \in \mathbb{R}^h$ as the complete time series of the variate indexed by $n$, collected from time point $t-h$ to $t-1$.

To generate $\mathbf{Y}_t$, we conceptualize the generation process as a node regression problem within the context of graph data. Specifically, each $\mathbf{X}_{t,n}$ in the input $\mathbf{X}_t$ is treated as a dynamic feature of node $n$ in a graph. We assume the existence of additional auxiliary node features $\mathbf{E}_{t, n}$. Consequently, the multivariate time series forecasting problem can be concisely formulated as follows:
\begin{equation}
\begin{split}
\mathbf{h}^0_{t,n} = &\text{Embedding}\left(\mathbf{X}_{t,n}, \mathbf{E}_{t,n}\right), \quad n=1,\ldots, N, \\
\mathbf{H}^{l+1}_{t} =& \text{GNN}\left(\mathbf{H}^{l}_{t}, \mathbf{A}^{l}\right), \quad l=1,\ldots, L, \\
\mathbf{Y}_{t, n} =& \text{Projection}\left(\mathbf{h}^L_{t,n}\right).
\end{split}
\end{equation}
Here, $\mathbf{H}_t = \{\mathbf{h}_{t,1}, \ldots, \mathbf{h}_{t,N}\} \in \mathbb{R}^{N \times D}$ represents $N$ node embeddings, each of dimension $D$, where the subscript $t$ denotes the time step and the superscript $l$ refers to the layer index. The term $\mathbf{A}$ represents the self-learned adjacency matrix. The whole ForecastGrapher framework is illustrated in Figure~\ref{fig:overall_model}.
In subsequent sections, we provide a detailed introduction to the design principles of $\text{Embedding}$ layer, $\text{GNN}$ layer, and $\text{Projection}$ layer. 

%Our ForecastGrapher framework, illustrated in Figure~\ref{fig:overall_model}, adheres to the node regression paradigm prevalent in GNNs. It incorporates key elements like node embedding, node projection, and GNN blocks with augmented expressive capabilities. Subsequent sections provide a detailed analysis of each component within the ForecastGrapher framework.

\begin{figure}[tb!]
\centerline{\includegraphics[width=0.8\linewidth]{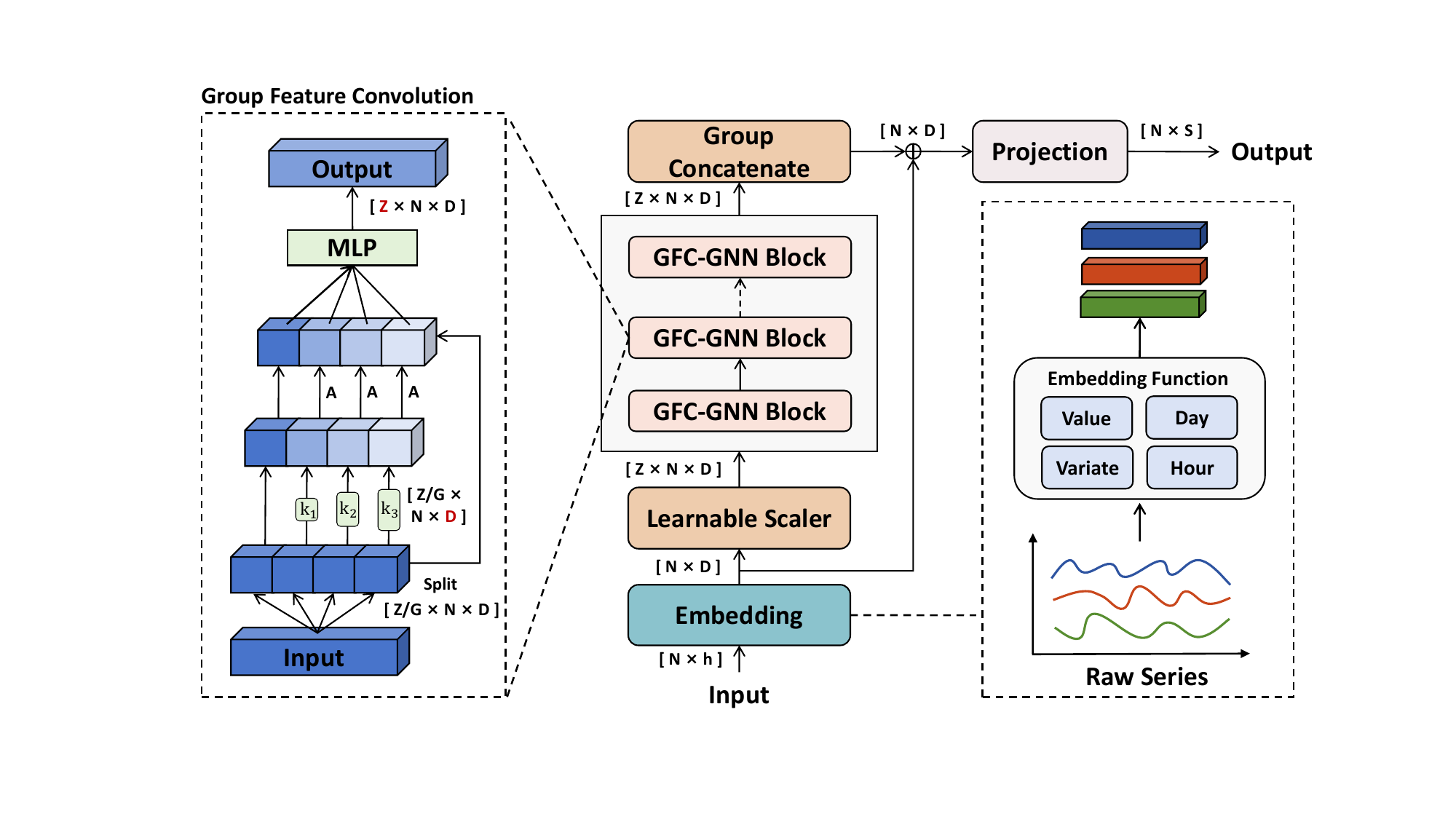}}
\caption{The overall structure of ForecastGrapher is designed to address a node regression task. The model considers each time series as a node and generates a corresponding node embedding. Next, it employs learnable scalers to partition the node embedding into multiple groups. Subsequently, several layers of GFC-GNN are stacked (the red color indicates the dimension to which the corresponding neural networks are applied). Finally, ForecastGrapher utilizes node projection for forecasting.}
\label{fig:overall_model}
\end{figure}

\subsection{Embedding the Time Series}

In our study, we treat each single variate as a node within a graph. The initial step involves integrating the temporal dynamics of time series into our node embeddings. There are many architectures like Temporal Convolutional Networks (TCN) and Transformers for this purpose.  Surprisingly, simpler linear models have demonstrated superior performance in capturing these temporal patterns~\cite{zeng2023transformers}. Therefore, in this research, we utilize a straightforward linear model to create embeddings that accurately reflect the temporal changes in individual time series. 

In various forecasting contexts, dynamic and static covariates that are known in advance play a significant role. Key among these are indicators related to "where" and "when."~\cite{shao2022spatial}. For instance, global covariates, common to all time series like time of day and day of the week, or specific ones such as a sensor's location, are crucial. In the datasets examined in our study, we have included three additional embeddings (variate, hour and day) to enrich the node embeddings. These embeddings are designed to capture and incorporate these essential covariate aspects effectively. In summary, the node embedding $\mathbf{h}^0_{t,n}$ is calculated by the following equation:
\begin{equation}
    \mathbf{h}^0_{t,n} = \text{Linear}(\mathbf{X}_{t,n}) + \mathbf{e}^{\text{variate}}_n + \mathbf{e}^{HiD}_{\pi(t)} + \mathbf{e}^{DiW}_{\pi(t)},
\end{equation}
where  $\text{Linear}$ denotes a straightforward linear layer. The term $\mathbf{e}^{\text{variate}}_n \in \mathbb{R}^{D}$ represents a learnable embedding associated with the $n$-th variate, $\mathbf{e}^{HiD}_{\pi(t)} \in \mathbb{R}^{D}$ and $\mathbf{e}^{DiW}_{\pi(t)} \in \mathbb{R}^{D}$ are learnable embeddings for Hour in Day and Day in Week, respectively, $\pi$ indicates the use of the corresponding granularity timestamp as the index for these temporal embedding matrices (In the experiments, if the data lacks "Hour in Day" and "Day in Week" information, we will not utilize these information).

\subsection{Graph Neural Networks}

\subsubsection{Self-Learnable Adjacency Matrix}

Defining an adjacency matrix for target time series is often challenging, and there are various methods to learn one from data. In our ForecastGrapher model, we employ the popular and straightforward method proposed by \cite{wu2019graph}. This involves two trainable parameters, $\mathbf{E}^l_1$ and $\mathbf{E}^l_2 \in \mathbb{R}^{N \times c}$, representing the source and target nodes respectively. The adjacency matrix is computed as follows:
\begin{equation}
\mathbf{A}^l = \text{SoftMax}\left(\text{ReLU}\left(\mathbf{E}^l_1 \left(\mathbf{E}^l_2\right)^T\right)\right),
\end{equation}
using the ReLU activation function to prune weak connections and the SoftMax function to normalize the adjacency matrix of the graph. It's important to highlight that in ForecastGrapher, we learn a new adjacency matrix at each layer. This strategy is implemented with the aim of capturing varying inter-series correlations across different layers.

\subsubsection{Learnable Scaler and Group Feature Convolution}

Departing from traditional human-crafted aggregators and scalers~\cite{corso2020principal}, our approach to enhancing the expressive power of GNNs involves a pure end-to-end strategy. This is inspired by the Bayesian neural network perspective of Convolutional Neural Networks (CNNs), which suggests that CNNs can autonomously modify the distribution of features~\cite{xiao2018dynamical,novak2018bayesian}. This insight forms the basis of our intuition that an automatic adjustment of feature distributions can be achieved within CNNs. We visualize the GFC mechanism in Figure~\ref{fig:GFC}.
\begin{wrapfigure}[16]{r}{7.5cm}
\centerline{\includegraphics[width=0.5\textwidth]{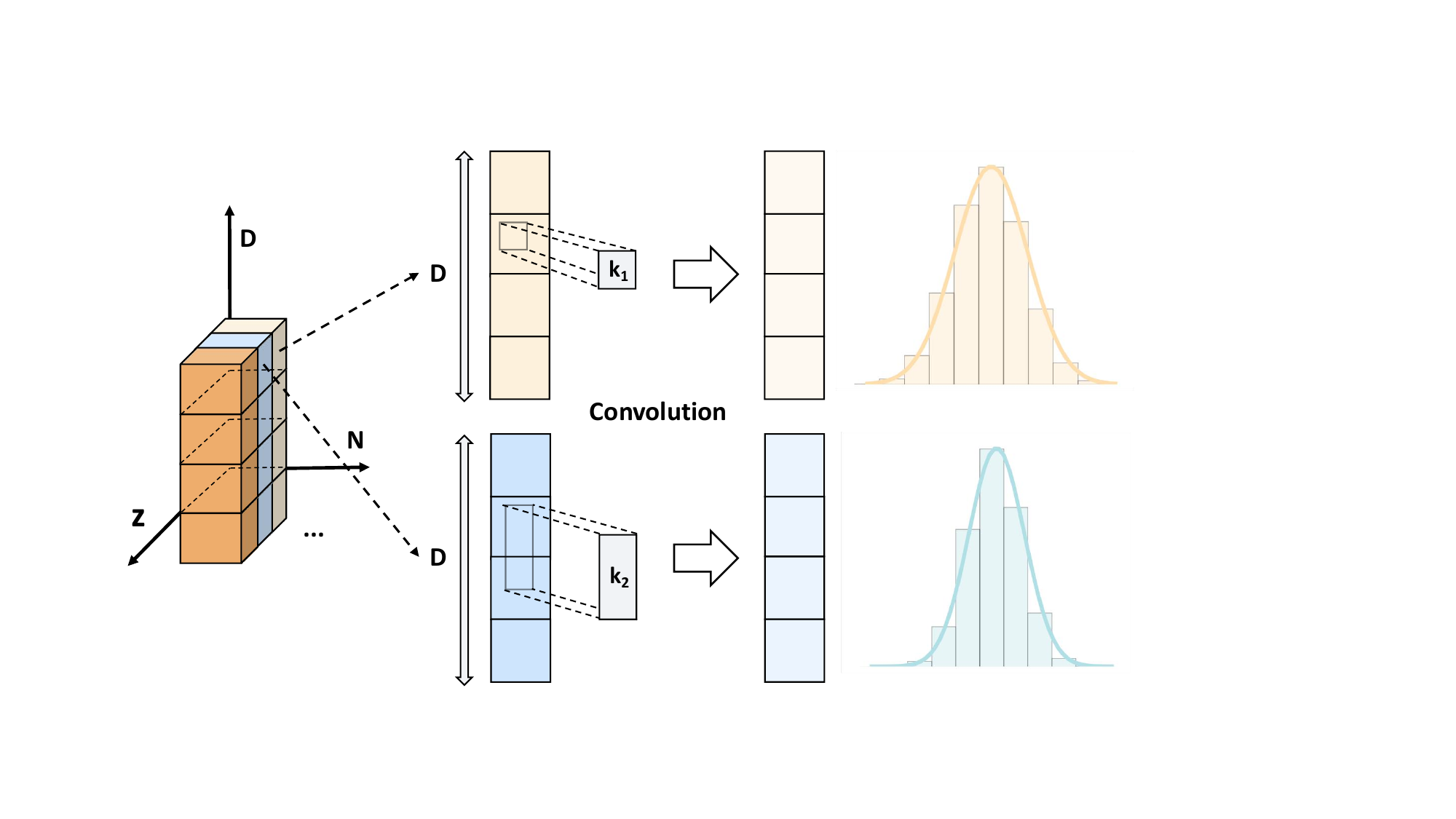}}
\caption{The GFC mechanism enhances the diversity of node embedding distributions: Convoluting the node feature with two distinct kernel lengths results in two distinct distributions.}
\label{fig:GFC}
\end{wrapfigure}

Initially, we augment the initial embedding $\mathbf{H}^0_{t}$ by multiplying it with $z$ learnable scalers. This process results in $\hat{\mathbf{H}}^0_{t} \in \mathbb{R}^{z \times N \times D}$, achieved by introducing a new dimension. In each GNN layer, we partition $\hat{\mathbf{H}}_{t}$ into $G$ groups, resulting in $\hat{\mathbf{H}}_{g,t} \in \mathbb{R}^{\frac{z}{G} \times N \times D}$, as illustrated in Figure~\ref{fig:overall_model}. Here, $g$ signifies the $g$-th group. For an integer division scenario, where $z \, \text{mod} \, G$ equals 0, we simply split into equal segments. Otherwise, we allocate $\text{int}\left(\frac{z}{G}\right)+z \, \text{mod} \, G$ dimensions to the first group, while maintaining $\text{int}\left(\frac{z}{G}\right)$ for the others. Subsequently, one-dimensional convolutions are applied on the feature dimension of $\hat{\mathbf{H}}_{g,t}$. To diversify the feature distribution of each group, different convolutional kernel sizes are used for each group, and one group is left unchanged without convolution. The formulation of GFC-GNN can be articulated as follows:
\begin{equation}
    \begin{split}
    \mathbf{T}^l_{g,t} &= \text{Conv1d}_g\left(\hat{\mathbf{H}}^l_{g,t}, k_g\right), \quad g = 2, \ldots, G, \\
    \mathbf{V}^l_{g,t} &=  \mathbf{A}^l \hat{\mathbf{T}}^l_{g,t}, \quad g = 2, \ldots, G, \\
    \mathbf{H}^{l+1}_{t} &= \text{MLP} \left( \hat{\mathbf{H}}^l_{1,t} \mid \mathbf{V}^l_{2,t} \mid \ldots \mid \mathbf{V}^l_{G,t} \right),
    \end{split}
    \label{eq:gfc}
\end{equation}
where $\text{Conv1d}_g$ represents the 1DCNN for the $g$-th group in the feature dimension, $k_g$ is the kernel length of the $g$-th group and $\text{MLP}$ denotes the multi-layer perceptron that is applied to the concatenated outputs of all groups. Notably, each 1DCNN is characterized by different kernel lengths, allowing for varied and specialized processing across different groups. We provide a theoretical analysis using Monte Carlo sampling to demonstrate that GFC-GNN can transform distributions with the same mean but different variances into distributions with different means, and it can generate diverse feature distributions, as shown in the Appendix~\ref{appen:mc}.

\subsection{Combining Groups and Generating Forecast Results}

After undergoing $L$ layers of graph convolution, we obtain the grouped representation $\mathbf{H}^{L-1}_{t}$ . To generate the forecast results via node regression, it's necessary to amalgamate the representations of the $G$ groups. To ensure that the representations of each group are fully utilized, we employ a learnable concatenation method. Additionally, we use a residual learning approach~\cite{he2016deep} to fuse the final representation with the initial representation. The finally representation is given by
\begin{equation}
    \mathbf{H}^{L}_{t} = \mathbf{H}^{0}_{t} + \mathbf{W}_{\text{concat}} \mathbf{H}^{L-1}_{t},
\end{equation}
where $\mathbf{W}_{\text{concat}} \in \mathbb{R}^{z \times 1}$ is the learnable weight matrix for group concatenation. 

Finally, the regression layer performs forecasting based on
\begin{equation}
    \mathbf{Y}_{t} =  \mathbf{H}^{L}_{t} \mathbf{W}_{\text{reg}} + \mathbf{b}_{\text{reg}},
\end{equation}
where $\mathbf{W}_{\text{reg}} \in \mathbb{R}^{D \times S}$ and $\mathbf{b}_{\text{reg}} \in \mathbb{R}^{S}$  are the learnable weights and bias, respectively.

\section{Experiments}
\subsection{Experimental Setup}
\paragraph{Datasets} We evaluate the performance of ForecastGrapher on twelve widely used datasets. These include ETT (h1, h2, m1, m2)~\cite{zhou2021informer}, Electricity, Exchange~\cite{lai2018modeling}, Traffic, Weather, and PEMS (03, 04, 07, 08), as evaluated in \cite{liu2023itransformer, liu2022scinet}. Additional information about these datasets can be found in the Appendix \ref{dataset_appendix}.

\paragraph{Baselines} We have selected several well-known forecasting models as our benchmarks, including \textbf{\emph{(\romannumeral1) Transformer-based models:}} iTransformer~\cite{liu2023itransformer}, PatchTST~\cite{Yuqietal-2023-PatchTST}, Crossformer~\cite{zhang2023crossformer}; \textbf{\emph{(\romannumeral2) MLP-based models:}} DLinear~\cite{zeng2023transformers}, RLinear~\cite{Li2023RevisitingLT}; \textbf{\emph{(\romannumeral3) TCN-based models:}} TimesNet~\cite{wu2023timesnet}, SCINet~\cite{liu2022scinet}. \textbf{\emph{(\romannumeral4) GNN-based models:}} FourierGNN\cite{yi2023fouriergnn}, StemGNN\cite{cao2020spectral}. Additionally, we include a Naive method, which repeats the last 24 values in the review window. %we use GCN~\cite{kipf2016semi}, GAT~\cite{velickovic2018graph}, and Mixhop~\cite{abu2019mixhop} to replace GFC-GNN in ForecastGrapher to validate our design. 

\paragraph{Parameter Settings} We maintained identical dataset partitioning and historical window length with $h=96$ following \cite{liu2023itransformer}. The prediction window length $S$ is set within $\{96,192,336,720\}$ or $\{12,24,48,96\}$. For the PEMS dataset, we selected $\{12,24,48,96\}$ time steps as the long-term prediction length, in contrast to short-term traffic forecasts, e.g., $\{3,6,12\}$ in traditional studies~\cite{li2018diffusion}. The batch size is fixed at $batch=32$, though it is reduced to $16$ in cases of insufficient memory. We limit the number of training epochs to $epochs=10$, using mean squared error (MSE) as the training loss function. Additional details on experimental parameters are provided in the Appendix \ref{app:Settings and Hyperparameters}. For the hardware, we employ 4 RTX 4090 24GB GPUs for our experiments.

\subsection{Forecasting Results}
The main prediction results are shown in Table~\ref{tab:avg_results}, and we compare the performance with the benchmarks using average MSE and MAE of all output lengths, which the lower, the better. The outcomes reveal that ForecastGrapher demonstrates outstanding performance across all datasets. \begin{wrapfigure}[14]{r}{7cm}
\centerline{\includegraphics[width=0.5\textwidth]{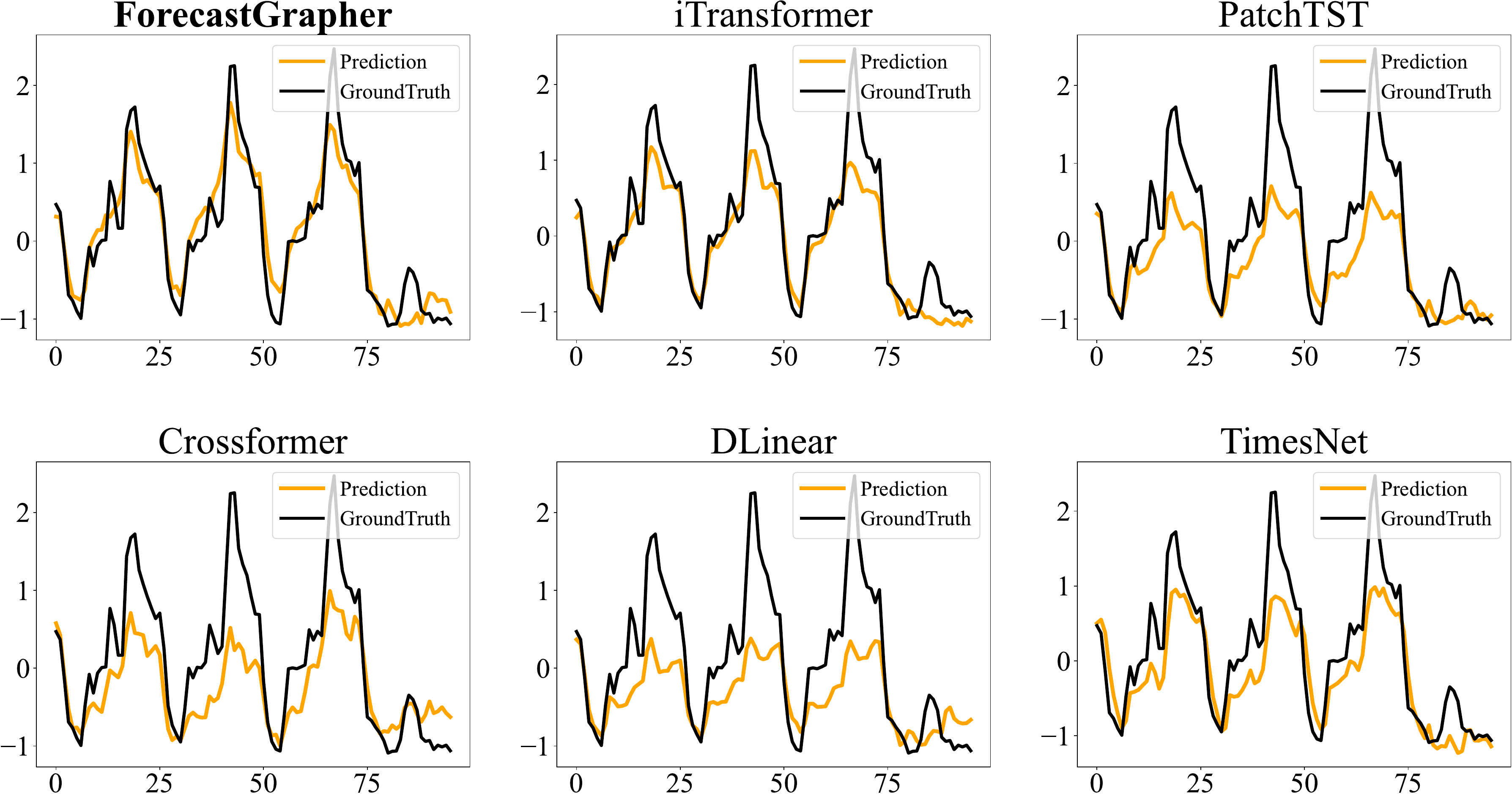}}
\caption{Visualization of input 96 and output 96 prediction results on the Electricity dataset.}
\label{fig:cases}
\end{wrapfigure}Specifically, ForecastGrapher achieves the top spot in terms of MSE and MAE a total of 16 times. Compared with the recent SOTA iTransformer, the error of ETT, Electricity, Weather, PEMS datasets significantly decreased by $4.21\%$, $7.25\%$, $4.47\%$ and $14.65\%$ respectively. Specifically, ForecastGrapher has demonstrated superior performance compared to other methods, notably iTransformer, on high-dimensional datasets such as Electricity and PEMS. iTransformer leverages the Transformer model to capture inter-series correlations by essentially creating a dynamic graph structure. Our research, however, suggests that GNNs utilizing static graph structures are more effective at capturing these inter-series correlations. Contrary to what one might intuitively expect, dynamic inter-series correlations may not provide benefits for high-dimensional, long-term forecasting tasks. To enable an intuitive comparison, we selected representative models for visualization on the Electricity dataset. As illustrated in Figure~\ref{fig:cases}, ForecastGraper exhibits a superior grasp of data fluctuations, surpassing other models in accuracy. When compared with other advanced GNNs and Naive method, ForecastGrapher also has a significant advantage in long-term prediction, as shown in Table \ref{tab:avg_results_gnn}.

\begin{table*}[htb!]
% \small
    \caption{Multivariate time series prediction results, with input length 96, output lengths in \{12,24,48,96\} for PEMS, \{96,192,336,720\} for others. Results are averaged from all output lengths. Use bold to indicate the best, and underline to indicate the second. The benchmarks are reported from \cite{liu2023itransformer}. Full results are listed in Appendix \ref{Full Forecasting Results}.}

    \centering
    \resizebox{1\linewidth}{!}{
    
    \begin{tabular}{cccccccccccccccccc}
\toprule
         \multicolumn{2}{c|}{Model}&  \multicolumn{2}{c}{\textbf{Ours}}&  \multicolumn{2}{c}{iTransformer}&  \multicolumn{2}{c}{PatchTST}&  \multicolumn{2}{c}{Crossformer}& \multicolumn{2}{c}{DLinear}& \multicolumn{2}{c}{RLinear} & \multicolumn{2}{c}{TimesNet}& \multicolumn{2}{c}{SCINet}\\
\midrule
         \multicolumn{2}{c|}{Metric}&  MSE&  MAE&  MSE&  MAE&  MSE&  MAE&  MSE& MAE& MSE& MAE& MSE& MAE& MSE& MAE& MSE& MAE\\
\toprule
         \multicolumn{2}{c|}{ETTm1}&  
\textbf{0.383} &  \textbf{0.397} &  0.407 &  0.410 &  {\ul 0.387} & {\ul 0.400} &  0.513 & 0.496 & 0.403 & 0.407 & 0.414 & 0.408 & 0.400 & 0.406 & 0.485 & 0.481 
\\
\midrule
         \multicolumn{2}{c|}{ETTm2}&  
\textbf{0.276} &  \textbf{0.323} &  0.288 &  0.332 &  {\ul 0.281} &  {\ul 0.326} &  0.757 & 0.610 & 0.350 & 0.401 & 0.286 & 0.327 & 0.291 & 0.333 & 0.571 & 0.537 
\\
\midrule

 \multicolumn{2}{c|}{ETTh1}& 
\textbf{0.437} & {\ul 0.437} & 0.454 & 0.447 & 0.469 & 0.454 & 0.529 & 0.522 & 0.456 & 0.452 & {\ul 0.446} & \textbf{0.434} & 0.458 & 0.450 & 0.747 &0.647 
\\
\midrule

 \multicolumn{2}{c|}{ETTh2}& 
\textbf{0.372} & {\ul 0.402} & 0.383 & 0.407 & 0.387 & 0.407 & 0.942 & 0.684 & 0.559 & 0.515 & {\ul 0.374} & \textbf{0.399} & 0.414 & 0.427 & 0.954 &0.723 
\\
\midrule

 \multicolumn{2}{c|}{Electricity}& 
\textbf{0.165} & \textbf{0.260} & {\ul 0.178} & {\ul 0.270} & 0.205 & 0.290 & 0.244 & 0.334 & 0.212 & 0.300 & 0.219 & 0.298 & 0.193 & 0.295 & 0.268 &0.365 
\\
\midrule

 \multicolumn{2}{c|}{Exchange}& 
0.367 & 0.407 & {\ul 0.360} & \textbf{0.403} & 0.367 & {\ul 0.404} & 0.940 & 0.707 & \textbf{0.354} & 0.414 & 0.378 & 0.417 & 0.416 & 0.443 & 0.750 &0.626 
\\
\midrule

 \multicolumn{2}{c|}{Traffic}& 
{\ul 0.458} & {\ul 0.292} & \textbf{0.428} & \textbf{0.282} & 0.481 & 0.304 & 0.550 & 0.304 & 0.625 & 0.383 & 0.626 & 0.378 & 0.620 & 0.336 & 0.804 &0.509 
\\
\midrule

 \multicolumn{2}{c|}{Weather}& 
\textbf{0.246} & \textbf{0.274} & {\ul 0.258} & {\ul 0.279} & 0.259 & 0.281 & 0.259 & 0.315 & 0.265 & 0.317 & 0.272 & 0.291 & 0.259 & 0.287 & 0.292 &0.363 
\\
\midrule

 \multicolumn{2}{c|}{PEMS03}& 
\textbf{0.098} & \textbf{0.205} & {\ul 0.113} & {\ul 0.221} & 0.180 & 0.291 & 0.169 & 0.281 & 0.278 & 0.375 & 0.495 & 0.472 & 0.147 & 0.248 & 0.114 &0.224 
\\
\midrule

 \multicolumn{2}{c|}{PEMS04}& 
{\ul 0.093} & {\ul 0.204} & 0.111 & 0.221 & 0.195 & 0.307 & 0.209 & 0.314 & 0.295 & 0.388 & 0.526 & 0.491 & 0.129 & 0.241 & \textbf{0.092} &\textbf{0.202} 
\\
\midrule

 \multicolumn{2}{c|}{PEMS07}& 
\textbf{0.079} &\textbf{0.172} & {\ul 0.101} & {\ul 0.204} & 0.211 & 0.303 & 0.235 & 0.315 & 0.329 & 0.395 & 0.504 & 0.478 & 0.124 & 0.225 & 0.119 &0.217 
\\
\midrule

 \multicolumn{2}{c|}{PEMS08}& 
\textbf{0.140} & \textbf{0.212} & {\ul 0.150} & {\ul 0.226} & 0.280 & 0.321 & 0.268 & 0.307 & 0.379 & 0.416 & 0.529 & 0.487 & 0.193 & 0.271 & 0.158 &0.244 
\\
\midrule
 \multicolumn{2}{c|}{$1^{st}$ Count}& \multicolumn{2}{c|}{\textbf{16}}& \multicolumn{2}{c|}{{\ul 3}}& \multicolumn{2}{c|}{0}& \multicolumn{2}{c|}{0}& \multicolumn{2}{c|}{1}& \multicolumn{2}{c|}{2}& \multicolumn{2}{c|}{0}& \multicolumn{2}{c}{2}\\

\bottomrule
    \end{tabular}

    }
    \label{tab:avg_results}
\end{table*}

\begin{table*}[htb!]
% \small
    \caption{Comparison with GNN and Naive method for multivariate time series prediction with input length 96. Results are averaged from all output lengths. Full results are listed in Appendix \ref{Full Forecasting Results}.}

    \centering
    \resizebox{1\linewidth}{!}{
    
    \begin{tabular}{cccccccccccccccc}
\toprule
         \multicolumn{2}{c|}{Dataset}&  \multicolumn{2}{c}{Electricity}&  \multicolumn{2}{c}{Traffic}&  \multicolumn{2}{c}{Weather}&  \multicolumn{2}{c}{PEMS03}& \multicolumn{2}{c}{PEMS04}& \multicolumn{2}{c}{PEMS07} & \multicolumn{2}{c}{PEMS08}\\
\midrule
         \multicolumn{2}{c|}{Metric}&  MSE&  MAE&  MSE&  MAE&  MSE&  MAE&  MSE& MAE& MSE& MAE& MSE& MAE& MSE& MAE\\
\toprule
         \multicolumn{2}{c|}{\textbf{Ours}}&  
\textbf{0.165} &  \textbf{0.260}
&  \textbf{0.458} &  \textbf{0.292 }
&  \textbf{0.246} & \textbf{0.274 }
&  \textbf{0.098} & \textbf{0.205 }
& \textbf{0.093} & \textbf{0.204 }
& \textbf{0.079} & \textbf{0.172 }
& \textbf{0.140} & \textbf{0.212} 
\\
\midrule
         \multicolumn{2}{c|}{FourierGNN}&  
0.228 &  0.324 
&  0.557 &  0.342 
&  0.249 &  0.302 
&  0.151 & 0.267 
& 0.180 & 0.294 
& 0.123 & 0.237 
& 0.216 & 0.312 
\\
\midrule

 \multicolumn{2}{c|}{StemGNN}& 
0.197 & 0.300 
& 0.612 & 0.356 
& 0.268 & 0.321 
& 0.187 & 0.302 
& 0.217 & 0.333 
& 0.184 & 0.289 
& 0.303 & 0.351 
\\

\midrule
 \multicolumn{2}{c|}{Naive}& 0.329 & 0.341 
& 1.156 & 0.477 
& 0.371 & 0.336 
& 0.901 & 0.703 
& 0.966 & 0.735 
& 0.966 & 0.720 
& 0.997 &0.745 
\\

\bottomrule
    \end{tabular}

    }
    \label{tab:avg_results_gnn}
\end{table*}

% \begin{figure}[tb]
% \centering % Centers the entire figure
% \begin{subfigure}[t]{0.49\textwidth} 
%   \centering % Centers the subfigure
%   \includegraphics[scale=0.11]{fig/cases.pdf} 
%   \caption{Visualization of input 96 and output 96 prediction results on the Electrcity dataset.}
%   \label{fig:cases}
% \end{subfigure}
% \hfill % Adds space between subfigures
% \begin{subfigure}[t]{0.49\textwidth}
%   \centering % Centers the subfigure
%   \includegraphics[scale=0.28]{fig/efficiency.pdf} 
%   \caption{}
%   \label{fig:efficiency}
% \end{subfigure}

% \caption{The perspective of node regression in multivariate forecasting and the limitations in the expressive power of conventional GCN mean aggregators.}
% \label{fig:gnn_model}
% \end{figure}

% \begin{figure}[t]
% \centerline{\includegraphics[scale=0.5]{fig/efficiency.pdf}}
% \caption{}
% \label{fig:efficiency}
% \end{figure}

\subsection{Ablation on GNNs}
We substitute our uniquely designed GNNs with simpler alternatives such as GCN~\cite{kipf2016semi}, GAT~\cite{velickovic2018graph}, and Mixhop~\cite{abu2019mixhop}. Details on how to implement other GNN variants are provided in the Appendix~\ref{app:gnn}. Additionally, we assess the impact of the GFC methodology on enhancing the performance of these GNN models by comparing their performance with and without the GFC component. 

The results on the Exchange, Traffic and Weather datasets are given in Table~\ref{tab:GNN_ablation}. The results indicate that the GFC module typically enhances the performance of all GNN models, particularly with more complex datasets. For instance, on the Traffic dataset, removing the GFC module led to a 7.59\% increase in error (from 0.461 to 0.496). In contrast, for simpler datasets like the Exchange, the \begin{wraptable}[13]{r}{7.7cm}
\caption{Ablation Study on GNNs and GFC: The prediction results are averaged across all prediction lengths.}
\centerline{
\resizebox{1\linewidth}{!}{
\small
\tabcolsep=0.18cm
\renewcommand\arraystretch{0.95}
\begin{tabular}{ccccccc}
\toprule
\multicolumn{1}{c}{Dataset}                                               & \multicolumn{2}{c}{Exchange}                                             &
\multicolumn{2}{c}{Traffic}                                               &
\multicolumn{2}{c}{Weather}  \\ \midrule

\multicolumn{1}{c}{Metric}    & MSE     & MAE    & MSE    & MAE     & MSE    & MAE             \\
 \toprule

\multicolumn{1}{c|}{ForecastGrapher}&0.367 &0.407 &0.458 &0.292 &0.246 &0.274\\
\midrule

\multicolumn{1}{c|}{GCN}  & \multicolumn{1}{c}{0.376}          & \multicolumn{1}{c}{0.412}          & \multicolumn{1}{c}{\textbf{0.462}} & \multicolumn{1}{c}{\textbf{0.293}} & \multicolumn{1}{c}{\textbf{0.245}} & \multicolumn{1}{c}{\textbf{0.275}} \\

\multicolumn{1}{c|}{GCN-w/o-GFC}   & \multicolumn{1}{c}{\textbf{0.370}} & \multicolumn{1}{c}{\textbf{0.409}} & \multicolumn{1}{c}{0.487}          & \multicolumn{1}{c}{0.328}          & \multicolumn{1}{c}{0.253}          & \multicolumn{1}{c}{0.279}          \\
\midrule

\multicolumn{1}{c|}{GAT}  & \multicolumn{1}{c}{\textbf{0.369}} & \multicolumn{1}{c}{\textbf{0.408}} & \multicolumn{1}{c}{\textbf{0.461}} & \multicolumn{1}{c}{\textbf{0.303}} & \multicolumn{1}{c}{\textbf{0.246}} & \multicolumn{1}{c}{\textbf{0.274}} \\

\multicolumn{1}{c|}{GAT-w/o-GFC}   & \multicolumn{1}{c}{0.375}          & \multicolumn{1}{c}{0.410}          & \multicolumn{1}{c}{0.496}          & \multicolumn{1}{c}{0.340}          & \multicolumn{1}{c}{0.260}          & \multicolumn{1}{c}{0.283}          \\

\midrule

\multicolumn{1}{c|}{Mixhop} & \multicolumn{1}{c}{\textbf{0.367}} & \multicolumn{1}{c}{\textbf{0.408}} & \multicolumn{1}{c}{\textbf{0.451}} & \multicolumn{1}{c}{0.294}          & \multicolumn{1}{c}{\textbf{0.248}} & \multicolumn{1}{c}{\textbf{0.275}} \\

\multicolumn{1}{c|}{Mixhop-w/o-GFC}   & \multicolumn{1}{c}{0.372}          & \multicolumn{1}{c}{0.411}          & \multicolumn{1}{c}{0.451}          & \multicolumn{1}{c}{\textbf{0.290}} & \multicolumn{1}{c}{0.249}          & \multicolumn{1}{c}{0.276}         \\ 
\bottomrule
\end{tabular}
}
}

 \label{tab:GNN_ablation}
\end{wraptable}differences with and without the GFC module are relatively minor. Interestingly, Mixhop aggregates multiple representations by performing multi-hop neighbor aggregation at each layer, which is akin to grouping using multi-hop neighbors at every layer. The performance does not fluctuate significantly after removing the GFC, suggesting that the grouping mechanism of GFC may not be fully compatible with Mixhop's approach of aggregating multi-hop neighbors. This discrepancy warrants further in-depth research in the future. In summary, all variants of GNNs demonstrate competitive performance, indicating that using GNNs for node regression is highly suitable for addressing multivariate time series forecasting problems.

\subsection{Ablation on Inter-series Correlation Learning Mechanism}

Our model primarily relies on self-learnable variate embeddings and a self-learnable adjacency matrix to capture inter-series correlations. The PEMS datasets provide adjacency matrices associated with actual distances, offering us a valuable opportunity to validate the effectiveness of the self-learnable adjacency matrix. To evaluate the influence of self-learnable variate embeddings and self-learnable adjacency matrices in ForecastGrapher, we design two distinct variants:
\begin{enumerate}
    \item 'w/o-variate' eliminates self-learnable variate embeddings of nodes.
    \item 'w/o-adp' replaces the adaptive adjacency matrix with the original distance-based adjacency matrix.
\end{enumerate}

The results on PEMS are presented in Table~\ref{tab:Emb_ablation2}. 
We have found that both self-learnable variate embeddings and adaptive adjacency matrices are equally important. For short-term prediction tasks (prediction length = 12), removing these components does not significantly reduce the model's performance. However, when the prediction horizon is extended, removing these components results in a much worse performance. This also indicates that distance-based adjacency matrices are insufficient to fully leverage the inter-series correlations within the traffic dataset. In addition, a visual analysis of the adjacency matrix can be seen in the Appendix~\ref{Learned Graph Visualization}.

%We believe that, over a wide range of time, dynamic graphs are challenging to guide model generate forecastings and might even introduce incorrect correlations. We employee a simple self-learning static adjacency matrix in the original model, and in the PeMS dataset, there is an adjacency matrix associated with the actual distance. We replaced the self-learning adjacency matrix with this static matrix. At the prediction length of 12, there are relatively minor differences between the two approaches , {\color{red}which is similar to the performance in short-term prediction.} As the prediction length further increases, the advantages of self-learning methods become more prominent. 

\begin{table}[htb!]
 \caption{Ablation on PEMS dataset, which includes ablation experiments of variable embedding and self-learning adjacency matrix.}
\centerline{
\resizebox{0.77\linewidth}{!}{
\small
\tabcolsep=0.18cm
\renewcommand\arraystretch{0.98}
\begin{tabular}{cccccccccc}
\toprule
\multicolumn{2}{c}{Dataset}                                               & \multicolumn{2}{c}{PEMS03}                                                &
\multicolumn{2}{c}{PEMS04}                                               &  
\multicolumn{2}{c}{PEMS07} &  \multicolumn{2}{c}{PEMS08}\\ \midrule
\multicolumn{2}{c}{Metric}    & MSE     & MAE    & MSE    & MAE     & MSE    & MAE       & MSE&MAE\\ 
\toprule

%12
\multicolumn{1}{c|}{
}  &\multicolumn{1}{c|}{ForecastGrapher}
& \textbf{0.065}       & \textbf{0.168}       & \textbf{0.075}       & \textbf{0.181}       & \textbf{0.058}       & \textbf{0.152}        & \textbf{0.081} &\textbf{0.184} 
\\

\multicolumn{1}{c|}{
}   &\multicolumn{1}{c|}{w/o-variate}
& 0.069                & 0.174                & 0.079                & 0.187                & 0.063                & 0.161                 & 0.083 &0.187 
\\

\multicolumn{1}{c|}{\multirow{-3}{*}{\rotatebox{90}{12}}}  &\multicolumn{1}{c|}{w/o-adp}
& 0.068                & 0.173                & 0.078                & 0.185                & 0.063                & 0.159                 & 0.086 &0.191 
\\
\midrule

% %24
% \multicolumn{1}{c|}{
% }  &\multicolumn{1}{c|}{ForecastGrapher}
% & \textbf{0.081}       & \textbf{0.186}       & \textbf{0.085}       & \textbf{0.194}       & \textbf{0.069}       & \textbf{0.163}        & \textbf{0.115} &\textbf{0.220} 
% \\

% \multicolumn{1}{c|}{
% }   &\multicolumn{1}{c|}{w/o-variate}
% & 0.089                & 0.198                & 0.094                & 0.207                & 0.077                & 0.177                 & 0.125 &0.230 
% \\

% \multicolumn{1}{c|}{\multirow{-3}{*}{\rotatebox{90}{24}}}  &\multicolumn{1}{c|}{w/o-adp}
% & 0.089                & 0.197                & 0.092                & 0.203                & 0.081                & 0.178                 & 0.128 &0.235 
% \\
% \midrule

% %48
% \multicolumn{1}{c|}{
% }  &\multicolumn{1}{c|}{ForecastGrapher}
% & \textbf{0.111}       & \textbf{0.220}       & \textbf{0.099}       & \textbf{0.213}       & \textbf{0.085}       & \textbf{0.179}        & \textbf{0.169} &\textbf{0.211} 
% \\

% \multicolumn{1}{c|}{
% }   &\multicolumn{1}{c|}{w/o-variate}
% & 0.120                & 0.234                & 0.115                & 0.231                & 0.097                & 0.199                 & 0.186 &0.233 
% \\

% \multicolumn{1}{c|}{\multirow{-3}{*}{\rotatebox{90}{48}}}  &\multicolumn{1}{c|}{w/o-adp}
% & 0.127                & 0.232                & 0.115                & 0.228                & 0.113                & 0.209                 & 0.183 &0.230 
% \\
% \midrule

%96
\multicolumn{1}{c|}{
}  &\multicolumn{1}{c|}{ForecastGrapher}
& \textbf{0.134}       & \textbf{0.244}       & \textbf{0.112}       & \textbf{0.227}       & \textbf{0.103}       & \textbf{0.194}        & \textbf{0.197} &\textbf{0.234} 
\\

\multicolumn{1}{c|}{
}   &\multicolumn{1}{c|}{w/o-variate}
& 0.158                & 0.269                & 0.132                & 0.250                & 0.119                & 0.220                 & 0.221 &0.265 
\\

\multicolumn{1}{c|}{\multirow{-3}{*}{\rotatebox{90}{96}}}  &\multicolumn{1}{c|}{w/o-adp}
& 0.152                & 0.257                & 0.137                & 0.250                & 0.147                & 0.238                 & 0.213 &0.253 
\\
\bottomrule

\end{tabular}
}
}

 \label{tab:Emb_ablation2}
\end{table}

\subsection{Extending the Historical Review Window and Model Efficiency}
\label{Extending the Historical Review Window and Model Efficiency}
Previous research~\cite{zeng2023transformers} has shown that a Transformer model may not necessarily be able to effectively extract information from longer review windows. It is natural to inquire whether the GNN-based ForecastGrapher can still capture sufficient temporal correlations from a longer review window. A robust model should exhibit improved performance as the review window extends, rather than displaying significant fluctuations.

To evaluate whether ForecastGrapher can leverage extended historical review windows, we conduct experiments on ETT, Weather, Traffic and Electricity datasets. The input length is varied from shorter to longer as $\{48,96,192,336,512\}$, and the model is assigned the task of forecasting the values for the next $96$ time steps. Figure~\ref{fig:longerseq} illustrates that ForecastGrapher effectively captures temporal correlations from these extended review windows, yielding superior results as the review windows lengthen, with minimal fluctuations toward optimal performance. This demonstrates that the node regression framework provided by ForecastGrapher remains robust when dealing with multivariate time series containing long sequence inputs. Additionally, we conduct a comprehensive comparison of the performance, training speed, and memory usage of ForecastGrapher and other models on Electricity, as shown in Figure~\ref{fig:efficiency}. Although ForecastGrapher does not achieve the best results in terms of training speed and memory usage, it still outperforms models like PatchTST and FourierGNN. While TimesNet has relatively low memory usage, its training speed is the slowest. Overall, our model achieves the best performance at an acceptable cost.

\begin{figure}[tb]
\centering % Centers the entire figure
\begin{subfigure}[t]{0.49\textwidth} 
  \centering % Centers the subfigure
  \includegraphics[scale=0.36]{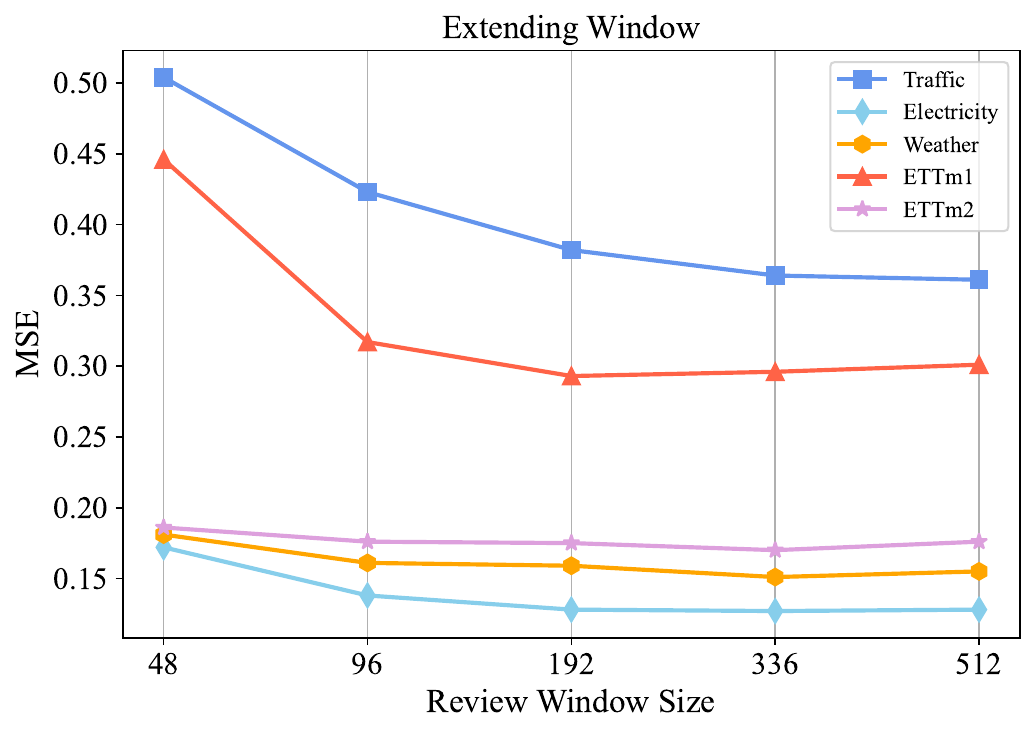} 
  \caption{Forecasting results with output length 96 and input length in \{48,96,192,336,512\}.}
  \label{fig:longerseq}
\end{subfigure}
\hfill % Adds space between subfigures
\begin{subfigure}[t]{0.49\textwidth}
  \centering % Centers the subfigure
  \includegraphics[scale=0.36]{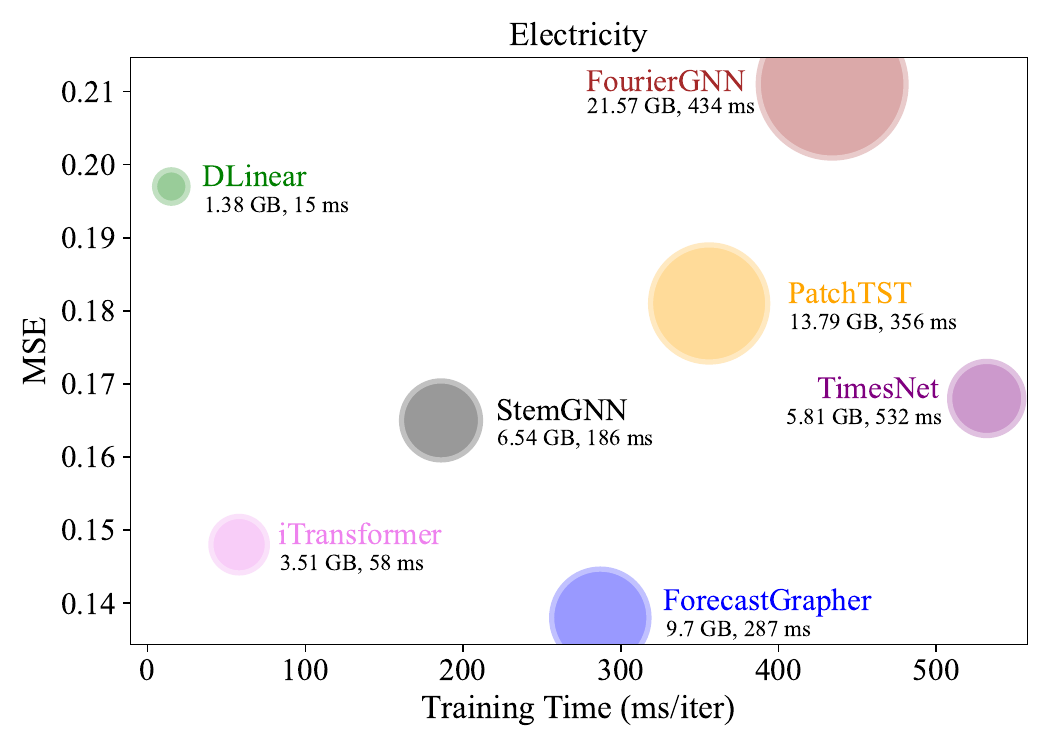} 
  \caption{Model efficiency comparison on Electricity with input length 96 and output length 96.}
  \label{fig:efficiency}
\end{subfigure}

\caption{Analysis of the model robustness and efficiency.}
\label{fig:robustness and efficiency}
\end{figure}

% \begin{figure}[t]
% \centerline{\includegraphics[scale=0.45]{fig/longerseq.pdf}}
% \caption{Forecasting results with 96 predict length and input length in \{48,96,192,336,512\}.}
% \label{fig:longerseq}
% \end{figure}

\subsection{Hyperparameter Sensitivity}
We evaluate the hyperparameter sensitivity of ForecastGrapher, focusing specifically on three factors: the embedding dimension $D$, the number of GNN layers $Layer$, and the learning rate $LR$. The results are presented in Figure \ref{fig:sensitivity}. We observe  that for datasets with large variables such as Traffic and Electricity, the error decreases as $D$, $Layer$, and $LR$ increase. The Weather dataset exhibits less sensitivity to these hyperparameters, maintaining stable performance across varying configurations.

\begin{figure}[t]
\centerline{\includegraphics[width=1\linewidth]{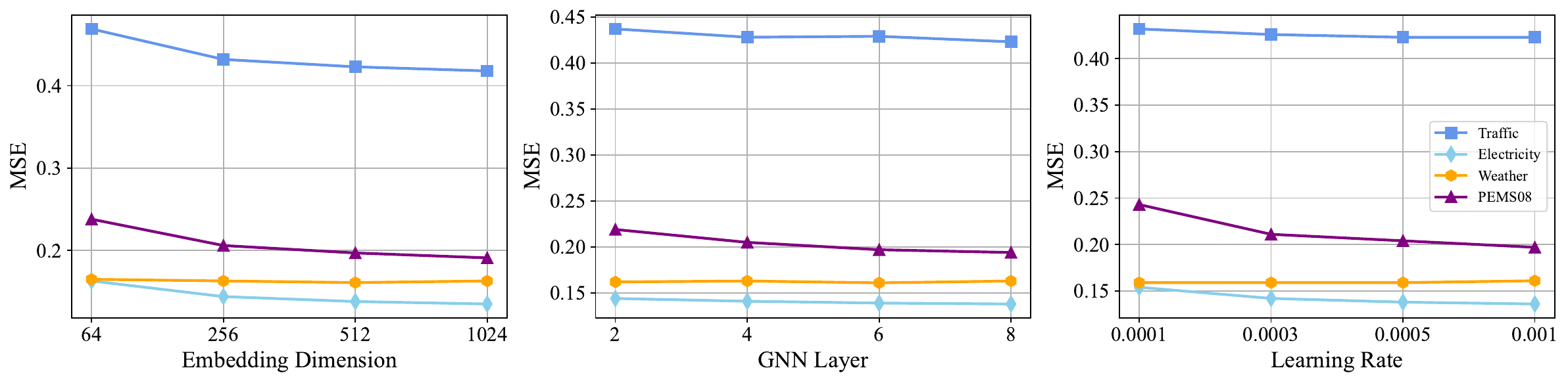}}
\caption{Sensitivity analysis in embedding dimension, number of GNN layers and learning rate. The results are recorded with input length 96 and output length 96.}
\label{fig:sensitivity}
\end{figure}

\section{Limitations}
\label{Limitations}
Our work has some limitations, as we have only focused on time series forecasting tasks within this framework. Research on other time series tasks, such as time series classification and anomaly detection, has not been conducted. Although ForecastGrapher outperforms models like iTransformer and DLinear in terms of prediction accuracy, its computational cost is higher. The main issue may lie in the self-learning graph structure, which still has room for optimization.

\section{Conclusion}

In this study, we take the innovative approach of framing multivariate time series forecasting as a node regression task within graph data, employing GNNs to address this challenge. We treat each variate in the dataset as a node, forming a graph that effectively captures the 'where' and 'when' information of the target time series. However, applying GNNs directly to this graph structure encounters limitations in expressive power. To overcome this, we introduce learnable scalars and 1D convolutions on the feature dimensions within each node to enhance information diversity. Leveraging this node regression framework and an enhanced GNN block, we develop the ForecastGrapher architecture. Through extensive testing across twelve datasets, the superiority of ForecastGrapher has been clearly demonstrated. We envision this pioneering effort as a foundational architecture for a broad range of time series analysis tasks.

%From the above equation, it is evident that CNNs with varying kernel sizes offer a convenient means of diversifying the feature distribution. Even if the input distributions are identical, different values of $k$ will result in entirely distinct output distributions. This implies that by merely adjusting a hyperparameter of the CNNs, we can obtain a variety of distributions, a feat that is challenging to achieve with other structures. 

%Another advantage of our architecture is the utilization of learnable scalars to alter the distribution of the input. Given a scalar \( a \), the transformed input \( x^{l}_{j}(\alpha) \) will be distributed as \( \mathcal{N}(a\mu_l, a^2\sigma_l^2) \), leading to a completely different distribution after the CNN projection.

%%%%%%
\normalem %add
\bibliographystyle{plain}
\bibliography{arxiv}{}
%%%%%%

\newpage
\appendix

\section{Deep Analysis of GFC-GNN}
\label{appen:mc}
We first analyze the limitation of conventional GCNs following~\cite{bi2023mm}. Our assumption is that there are two node classes on the graph, which exhibit significant differences in the distribution of their future values. However, their node features share similarities in distribution, such as having equal means. For a standard GCN using weighted mean aggregator~\cite{kipf2016semi}, we identify the following limitations. 
\begin{theorem}
Given a graph $\mathcal{G}(\mathcal{V}, \mathcal{E})$, we denote the nodes belonging to class $C_i$ as $\{v_i \mid v_i \in C_i\}$. Assume the feature distribution $\mathbf{h}_{i}$ of nodes in $C_i$ follows an i.i.d. Gaussian distribution $\mathcal{N}(\mu_i, \sigma_i^2)$. For any two distinct classes $C_i$ and $C_j$, if $\mu_i = \mu_j$ and $\sigma_i \neq \sigma_j$, then the $p$-norm distance between the expectation of GCN outputs of these two classes is zero: $\left\| \mathbb{E}_{v_k \sim C_i} \left(\text{GCN}\left(\mathbf{h}_k\right)\right) - \mathbb{E}_{v_k \sim C_j} \left(\text{GCN}\left(\mathbf{h}_k\right)\right) \right\|_p = 0$.
\end{theorem}

\begin{proof}
To simplify the analysis, we omit the activation function and learnable weights. The update function of a GCN with a mean aggregator is:
\begin{equation}
    \mathbf{h}^{(l+1)}_k = \frac{1}{d_k} \sum_{n \in \mathcal{N}(v_k)} a_{nk} \mathbf{h}^{(l)}_n,
\end{equation}
where $d_k$ is the degree of node $k$, defined as $d_k = \sum_{n \in \mathcal{N}(v_k)} a_{nk}$. Here, $a_{nk}$ represents the element of the adjacency matrix $\mathbf{A}$.  Therefore, we have:
\begin{equation}
 \mathbb{E}_{v_k \sim C_i} \left(\mathbf{h}^{(l+1)}_k\right) = \mathbb{E}_{v_k \sim C_j} \left(\mathbf{h}^{(l+1)}_k\right).
\end{equation}
\end{proof}

Now we analyze the shift in nodes' feature distribution of GFC. We focus on a simplified scenario in which the 1D CNN employs circularly-padded activations. Additionally, both the weights of this CNN and the nodes' features prior to convolution are assumed to be independently and identically distributed (i.i.d.) and drawn from a Gaussian distribution (similar to \cite{xiao2018dynamical, novak2018bayesian}).

Let $\mathbf{h}^{l}(\alpha)$ denote the output at layer $l$ and spatial location $\alpha$. Assume $\mathbf{h}^{l}(\alpha)$ is independently and identically distributed (i.i.d.) from the Gaussian distribution $\mathcal{N}\left(\mu, \sigma^2\right)$. For the $j$-th value in the given group $g$, we scale $\mathbf{h}^{l}(\alpha)$ by a factor of $s_j$ to obtain $\mathbf{h}_j^{l}(\alpha)$. Consequently, $\mathbf{h}_j^{l}(\alpha)$ follows a distribution of $\mathcal{N}\left(s_j\mu, s_j^2\sigma^2\right)$. Consider a 1D periodic CNN with a filter size of $k_g$, a channel size of $\frac{z}{G}$, and a spatial size of $D$ (where convolution is performed over the feature dimension of size $D$). Assume the weights $\mathbf{W}^{l}_{g} \in \mathbb{R}^{k_g \times \frac{z}{G} \times \frac{z}{G}}$ are i.i.d. from $\mathcal{N}\left(\mu_w, \sigma_w^2\right)$, and we have $\text{ACT}$ as the activation function.
The forward-propagation dynamics are described by the following recurrence relation:
\begin{equation}
\mathbf{h}_i^{l+1}(\alpha) = \text{ACT}\left(\sum^{\frac{z}{G}}_{j=1}\sum^{k_g}_{\beta = 1} \mathbf{h}_j^{l}\left(\alpha + \beta\right) \mathbf{W}^{l}_{g}(\beta, i, j) \right).
\end{equation}
%We will see that the resulting distribution of $\mathbf{h}_i^{l+1}(\alpha)$ corresponds to the distribution for sums of products of Gaussian variables. The mean of this distribution is heavily impacted by factors such as the mean value and standard deviation of $\mathbf{h}_j^{l}\left(\alpha + \beta\right)$, as well as the number of additions~\cite{luo2016understanding}, and is therefore significantly influenced by both the kernel length $k_g$ and learnable scaler $s_j$.

\begin{figure}[t]
\centering % Centers the entire figure

\begin{subfigure}[b]{0.45\textwidth} 
  \centering % Centers the subfigure
  \includegraphics[scale=0.25]{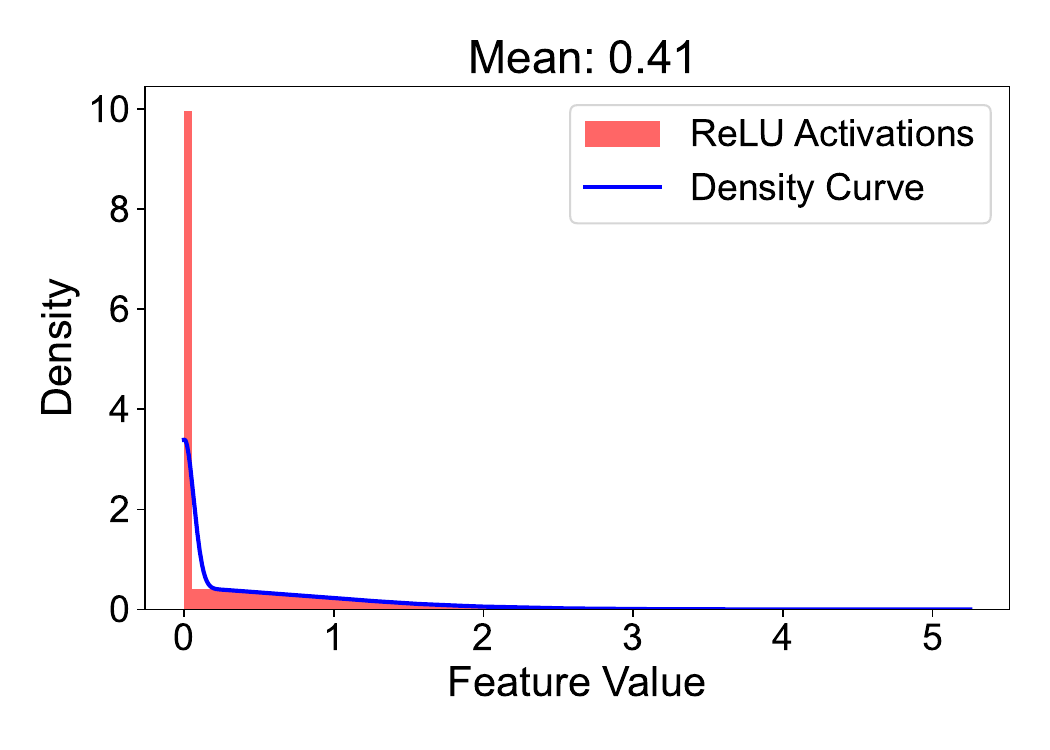} 
  \caption{Distribution shift for $\mathcal{N}(1, 1)$ with $k_g = 2$, $\frac{z}{G} =2$, $s_1 = -0.25$, $s_2 = 0.25$ }
  \label{fig:s1}
\end{subfigure}
% \hfill % Adds space between subfigures
\begin{subfigure}[b]{0.45\textwidth}
  \centering % Centers the subfigure
  \includegraphics[scale=0.25]{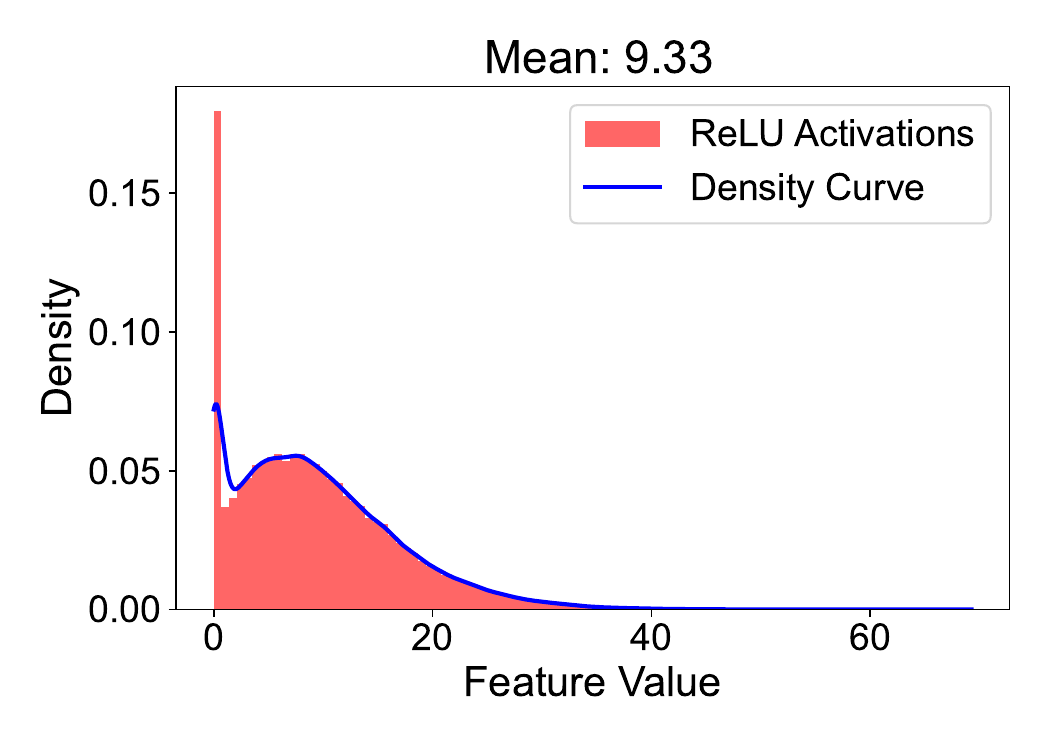} 
  \caption{Distribution shift for $\mathcal{N}(1, 1)$ with $k_g = 3$, $\frac{z}{G} =2$, $s_1 = 2.0$, $s_2 = 2.5$ }
  \label{fig:s2}
\end{subfigure}
\begin{subfigure}[b]{0.45\textwidth} 
  \centering % Centers the subfigure
  \includegraphics[scale=0.25]{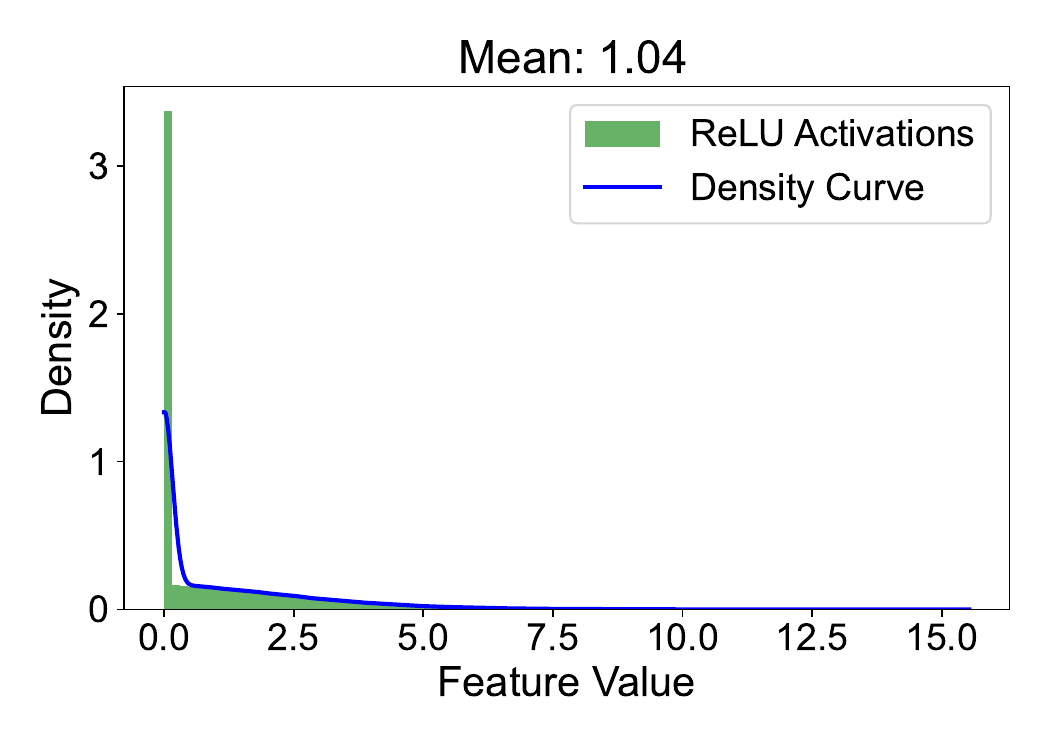} 
  \caption{Distribution shift for $\mathcal{N}(1, 3)$ with $k_g = 2$, $\frac{z}{G} =2$, $s_1 = -0.25$, $s_2 = 0.25$}
  \label{fig:s3}
\end{subfigure}
% \hfill % Adds space between subfigures
\begin{subfigure}[b]{0.45\textwidth}
  \centering % Centers the subfigure
  \includegraphics[scale=0.25]{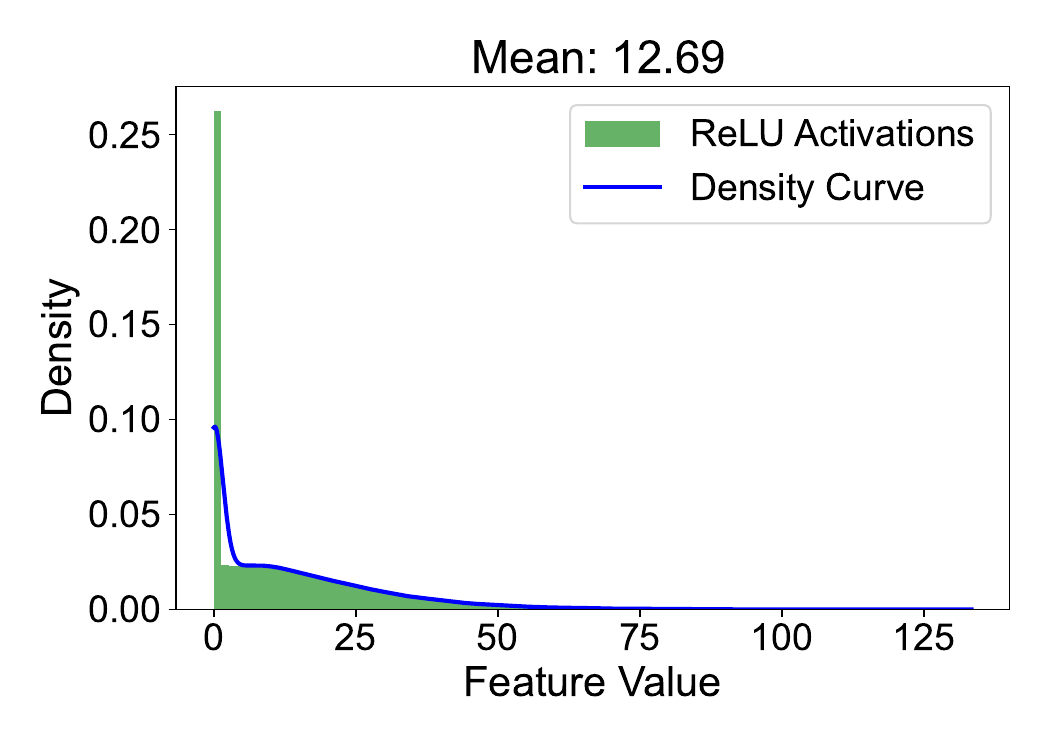} 
  \caption{Distribution shift for $\mathcal{N}(1, 3)$ with $k_g = 3$, $\frac{z}{G} =2$, $s_1 = 2.0$, $s_2 = 2.5$ }
  \label{fig:s4}
\end{subfigure}
\caption{Monte Carlo simulations of the distribution and mean value of $\mathbf{h}_i^{l+1}(\alpha)$.}
\label{fig:conv_distribution}
\end{figure}

Providing an analytical form for $\mathbf{h}_i^{l+1}(\alpha)$ is challenging, but its mean and distribution depend not only on the expected value $\mu$ and variance $\sigma$ of the input distribution but also on the convolution kernel length $k_g$ and the value of the learnable scaler $s_j$. We utilized the Monte Carlo method to analyze the distribution of input features following the distributions $\mathcal{N}(1, 1)$ and $\mathcal{N}(1, 3)$ under the influence of two different sets of kernel lengths $k_g$ and learnable scalers. We set the activation function as ReLU~\cite{fukushima1975cognitron} function. 
The results are illustrated in Figure~\ref{fig:conv_distribution}. From the figure, it is evident that even if the distributions of $\mathbf{h}^{l}(\alpha)$ and the learned 1D CNN weights are identical, varying the values of $k_g$ and $s_j$ can lead to entirely distinct output distributions. Moreover, when distributions with the same mean but different variances undergo group feature convolution, the expected value of the feature also changes.% Thus, the GFC-GNN model can circumvent the issue with GCNs described in Theorem 1. We include additional information in the Appendix \ref{Analyzing Distribution Shifts Post-GFC}

\section{Implementation Details}
\label{Implementation Details}

\subsection{Datasets}
\label{dataset_appendix}
We utilized 12 datasets in our experiment, all of which are extensively employed for benchmark testing. The datasets encompass a diverse range of applications and scenarios, ensuring a comprehensive evaluation of our method. The following provides a detailed overview of each dataset: (1) \textbf{ETT} \cite{zhou2021informer} collects 7 features data at two distinct time scales: hourly and every 15 minutes. These data are gathered from two regions, resulting in a total of four datasets: h1, h2, m1, and m2. (2) \textbf{Electricity}\footnote{https://archive.ics.uci.edu/ml/datasets/ElectricityLoadDiagrams20112014 }  records the hourly electricity consumption of 321 customers. (3) \textbf{Exchange}\cite{lai2018modeling} records daily exchange rates for 8 countries from 1990 to 2016. (4) \textbf{Traffic}\footnote{http://pems.dot.ca.gov} collects the road occupancy rate measured by 862 sensors on San Francisco freeways every hour since January 2015. (5) \textbf{Weather}\footnote{https://www.bgc-jena.mpg.de/wetter/} gathers 21 meteorological indicators, including air temperature, with a ten-minute time granularity. (6) \textbf{PEMS} collects traffic flow data in California through multiple sensors and we use four datasets including 03, 04, 07, 08 used by iTransformer\cite{liu2023itransformer}. 

We set the input length to $96$, the output lengths for the PEMS dataset are $\{12, 24, 48, 96\}$, and for others are $\{96,192,336,720\}$. Table~\ref{tab:dataset} presents the number of variate, prediction length, dataset partition size, and frequency information for each dataset, providing a overview of the datasets used in our experiment. This information is essential for understanding the scale and characteristics of the datasets. %In Table~\ref{tab:dataset}, we also offer the data standardization option for different datasets. 

\begin{table}[htb!]
\caption{Description of all datasets.}
\renewcommand{\arraystretch}{1} 
\centerline{
\resizebox{0.7\linewidth}{!}{
\begin{tabular}{c|cccc}
\toprule
\multirow{1}{*}{Datasets} & \multirow{1}{*}{Nodes}  &\multirow{1}{*}{Prediction Length}& \multirow{1}{*}{Dataset Size}& \multirow{1}{*}{Frequency} \\
                          \midrule
ETTm1                    & 7                       &\{96, 192, 336, 720\}& (34465, 11521, 11521)         & 15 minutes                 \\ \midrule
ETTm2                    & 7                       &\{96, 192, 336, 720\}& (34465, 11521, 11521)         & 15 minutes                 \\ \midrule
ETTh1                    & 7                       &\{96, 192, 336, 720\}& (8545, 2881, 2881)            & Hourly                     
\\ \midrule

ETTh2                    & 7                       &\{96, 192, 336, 720\}& (8545, 2881, 2881)            & Hourly                     
\\ \midrule
Electricity              & 321                     &\{96, 192, 336, 720\}& (18317, 2633, 5261)           & Hourly                     \\ \midrule
Exchange                 & 8                       &\{96, 192, 336, 720\}& (5120, 665, 1422)             & Daily                      \\ \midrule
Traffic& 862 &\{96, 192, 336, 720\}& (12185,1757,3590)& Hourly                     \\ \midrule 
 Weather                  & 21                      &\{96, 192, 336, 720\}& (36792, 5271, 10540)          &10 minutes                 \\ \midrule
 PEMS03& 358 &\{12, 24, 48, 96\}& (15629, 5147, 5147)&5 minutes\\ \midrule
 PEMS04& 307 &\{12, 24, 48, 96\}& (10100, 3303, 3304)&5 minutes\\ \midrule
 PEMS07& 883 &\{12, 24, 48, 96\}& (16839, 5550, 5550)&5 minutes\\ \midrule
 PEMS08& 170 &\{12, 24, 48, 96\}& (10618, 3476, 3477)&5 minutes\\ 
 \bottomrule
\end{tabular}
}
}
\label{tab:dataset}
\end{table}

\subsection{Settings and Hyperparameters}
\label{app:Settings and Hyperparameters}
\begin{table*}[h]
\caption{Hyperparameters of ForecastGrapher on different datasets.}
\renewcommand{\arraystretch}{1.2} 
\centerline{
\resizebox{1\linewidth}{!}{
\begin{tabular}{c|cccccccccccc}
\toprule
\multirow{1}{*}{Datasets} & \multirow{1}{*}{ETTm1}&\multirow{1}{*}{ETTm2} &\multirow{1}{*}{ETTh1} &\multirow{1}{*}{ETTh2}&\multirow{1}{*}{ECL}& \multirow{1}{*}{Exchange}& \multirow{1}{*}{Weather} &\multirow{1}{*}{Traffic} & \multirow{1}{*}{PEMS03}& \multirow{1}{*}{PEMS04}& \multirow{1}{*}{PEMS07}&\multirow{1}{*}{PEMS08}
\\
     \midrule
Epochs& \multicolumn{12}{c}{10}\\ \midrule

Batch& \multicolumn{7}{c|}{32}& \multicolumn{1}{c|}{16} & \multicolumn{4}{c}{32}\\
\midrule
Loss& \multicolumn{12}{c}{MSE}\\ \midrule

Learning Rate& \multicolumn{1}{c|}{1e-4} &\multicolumn{1}{c|}{1e-4}  &\multicolumn{1}{c|}{1e-4} &\multicolumn{1}{c|}{1e-4}& \multicolumn{1}{c|}{5e-4}& \multicolumn{1}{c|}{1e-4}&\multicolumn{1}{c|}{1e-4}& \multicolumn{1}{c|}{1e-3} & \multicolumn{1}{c|}{1e-3}& \multicolumn{1}{c|}{5e-4}& \multicolumn{1}{c|}{1e-3}&1e-3
\\ \midrule

%GNN Block& \multicolumn{1}{c|}{\{1, 3\}} & \multicolumn{1}{c|}{4}& \multicolumn{1}{c|}{1}& \multicolumn{1}{c|}{3}& 4\\ \midrule

Layer& \multicolumn{1}{c|}{1} &  \multicolumn{1}{c|}{1}& \multicolumn{1}{c|}{2}&\multicolumn{1}{c|}{1}& \multicolumn{1}{c|}{8}& \multicolumn{1}{c|}{1}& \multicolumn{1}{c|}{6}&\multicolumn{1}{c|}{8} & \multicolumn{1}{c|}{6}&\multicolumn{1}{c|} {9}& \multicolumn{1}{c|}{6}&6
\\ \midrule

$D$& \multicolumn{1}{c|}{512} &  \multicolumn{1}{c|}{512}& \multicolumn{1}{c|}{128}&\multicolumn{1}{c|}{512}& \multicolumn{1}{c|}{512}& \multicolumn{1}{c|}{128}& \multicolumn{1}{c|}{512}& \multicolumn{1}{c|}{512} & \multicolumn{4}{c}{512}\\ \midrule

 $z$& \multicolumn{1}{c|}{32} &  \multicolumn{1}{c|}{32}& \multicolumn{1}{c|}{32}&\multicolumn{1}{c|}{32}& \multicolumn{1}{c|}{8}& \multicolumn{1}{c|}{8}& \multicolumn{1}{c|}{32}& \multicolumn{1}{c|}{8} & \multicolumn{4}{c}{8}\\ \midrule
 
 $c$& \multicolumn{7}{c|}{10}& \multicolumn{1}{c|}{1000} & \multicolumn{4}{c}{10}\\
 \midrule
 $k_g$& \multicolumn{12}{c}{3, 5, 7}\\
 \midrule
 group& \multicolumn{12}{c}{4}\\ \midrule
 
 %Norm& \multicolumn{5}{c}{1}\\ \midrule
Optimizer                   & \multicolumn{12}{c}{Adam \cite{kingma2014adam}}\\ \bottomrule
\end{tabular}
}
}
\label{tab:HyperPara}
\end{table*}

All experiments are performed utilizing RTX 4090 24GB GPU devices, and the training process is refined through the Adam\cite{kingma2014adam} optimizer, with MSE loss function. Regarding $batch$ size, 16 is chosen for the traffic dataset due to memory constraints, while 32 is maintained for all others. We set the embedding dimension $D$ within the range of $\{128, 512, 1024\}$, the learning rate within $\{0.0001, 0.0005, 0.001\}$, and the number of GNN layers from $1$ to $9$. The scaling number $z$ in the learnable scaler is selected from $\{8, 32\}$. We perform a grid search within these parameter ranges to find the optimal settings. Furthermore, we offer different standardization~\cite{kim2021reversible,liu2022non} options for different datasets, and we forego standardization for PEMS. For the PEMS datasets, we do not perform normalization prior to embedding. However, for all other datasets, normalization is applied beforehand. We provide specific hyperparameters for different datasets in Table~\ref{tab:HyperPara}. 

\section{Other GNN Variants}
\label{app:gnn}

For the standard GCN, we apply graph convolution to each group. The formula undergoes only minor changes compared to Equation~\eqref{eq:gfc}:
\begin{equation}
    \begin{split}
    \mathbf{T}^l_{g,t} &= \text{Conv1d}_g\left(\hat{\mathbf{H}}^l_{g,t}, k_g\right), \quad g = 2, \ldots, G, \\
    \mathbf{V}^l_{g,t} &=  \mathbf{A}^l \hat{\mathbf{T}}^l_{g,t}, \quad g = 1, \ldots, G, \\
    \mathbf{H}^{l+1}_{t} &= \text{MLP} \left( \hat{\mathbf{V}}^l_{1,t} \mid \mathbf{V}^l_{2,t} \mid \ldots \mid \mathbf{V}^l_{G,t} \right).
    \end{split}
\end{equation}

The only difference is that we do not retain a module that bypasses the graph convolution. If we do not use GFC, the model abandons the grouping mechanism and operates directly on the ungrouped $\mathbf{H}^l_{t}$. In this case, $\mathbf{H}^{l+1}_{t} = \mathbf{A}^l \mathbf{H}^l_{t} \mathbf{W}^l_{\text{GCN}}$. When using GAT and MixHop, we replace $\mathbf{V}^l_{g,t} =  \mathbf{A}^l \hat{\mathbf{T}}^l_{g,t}$ with the corresponding networks.

\section{Ablation on Embedding}

To assess the impact of different embedding components in ForecastGrapher, we develop three distinct variants:
\begin{enumerate}
    \item 'w/o-variate' eliminates variate embeddings of nodes, removing information about 'where.'.
    \item 'w/o-hid' omits hour in day embeddings, thus eliminating temporal information about 'when' within a single day.
    \item 'w/o-diw' removes day in week embeddings, which eliminates information about the 'when' across a week.
\end{enumerate}

The results on ETTm1, ETTm2 and ETTh2 (The ETT datasets include comprehensive calendar information) presented in Table \ref{tab:Emb_ablation} generally show that models achieve optimal performance when all embedding components are incorporated. However, it's important to note that not all datasets respond equally to various embedding strategies. For example, ETTm1 exhibits robustness even in the absence of variate and hid embeddings.

\begin{table}[htb!]
\caption{Ablation on embedding. The results are obtained from the mean of all prediction lengths, and the best results are highlighted in bold.}
\centerline{
% \resizebox{1\linewidth}{!}{
\small
\tabcolsep=0.18cm
\renewcommand\arraystretch{0.98}
\begin{tabular}{cccccll}
\toprule
\multicolumn{1}{c}{Dataset}                                               & \multicolumn{2}{c}{ETTm1}                                                &
\multicolumn{2}{c}{ETTm2}                                                &                                                \multicolumn{2}{c}{ETTh2}\\ \midrule
\multicolumn{1}{c}{Metric}    & MSE     & MAE    & MSE    & MAE     & MSE    &MAE      \\ 
                      \toprule
\multicolumn{1}{c|}{ForecastGrapher
}  & \textbf{0.383} & \textbf{0.397} & \textbf{0.276} & \textbf{0.323} & \textbf{0.372} & \textbf{0.402} \\

\multicolumn{1}{c|}{w/o-variate
}   & 0.384          & 0.397          & 0.280          & 0.325          & 0.377          & 0.403          \\

\multicolumn{1}{c|}{w/o-hid
}  & 0.384          & 0.398          & 0.278          & 0.324          & 0.376          & 0.404          \\

\multicolumn{1}{c|}{w/o-diw
}   & 0.395          & 0.408          & 0.278          & 0.324          & 0.384          & 0.407         
\\
\bottomrule
\end{tabular}
% }
}
 
 \label{tab:Emb_ablation}
\end{table}

\section{Learned Graph Visualization}
\label{Learned Graph Visualization}
To enhance the interpretability of our analysis, we visualize the learned partial adjacency matrix on the PEMS04 dataset in Figure \ref{fig:adj}. Specifically, we generate a heatmap to visualize the associations among the top 50 nodes in the dataset. This heatmap provides a quick overview of the relationships between these nodes. We also visualized the time series corresponding to node pairs with higher values in the learnable adjacency matrix. Moreover, we compare the learned adjacency matrix with the preset distance-based adjacency matrix. Firstly, the adjacency matrix learned by ForecastGrapher is sparser, indicating that the model requires fewer inter-series correlations for predictions. Furthermore, the learnable adjacency matrix captures connections between time series that are distant yet exhibit strong similarities, for example, nodes 0 and 36, 8 and 35, as well as 34 and 44, which are challenging to represent in distance-based static adjacency matrices. 
%It reveals a strong correlation between specific node pairs: nodes $0$ and $36$, $8$ and $35$, as well as $34$ and $44$, which are difficult to reflect in distance based static adjacency matrices.

\begin{figure*}[ht]
\centerline{\includegraphics[scale=0.5]{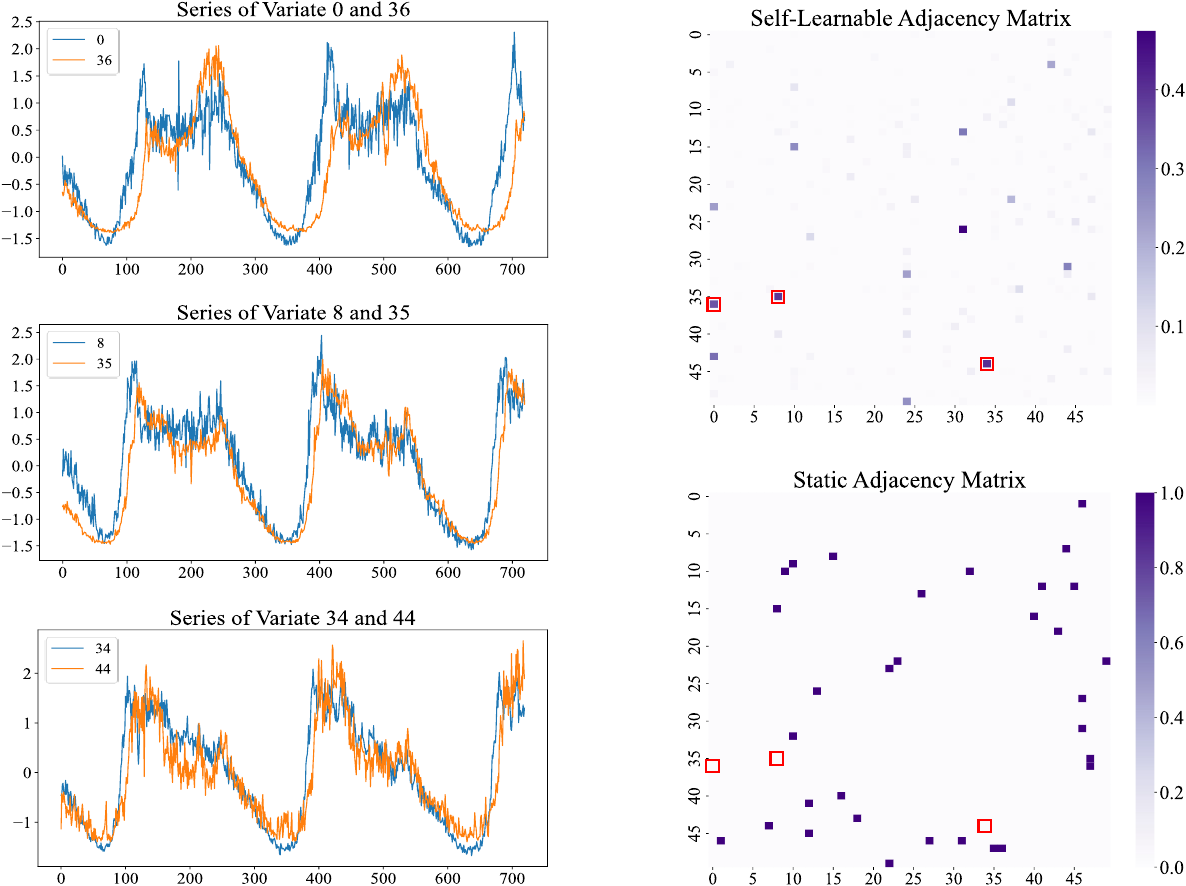}}
\caption{Visualization of the adjacency matrix for the top 50 nodes in the PEMS04 dataset, showcasing both the third-layer learnable adjacency matrix and the preset static adjacency matrix. The preset static adjacency matrix fails to capture the correlations between time series with strong similarities.}
\label{fig:adj}
\end{figure*}

\section{More Forecasting Results Visualization}

To facilitate a comprehensive comparison of model performances, we include additional prediction outcomes in Figure \ref{fig:traffic_appendix} to Figure \ref{fig:weather_appendix}. We utilize iTransformer\cite{liu2023itransformer}, PatchTST\cite{Yuqietal-2023-PatchTST},  Crossformer\cite{zhang2023crossformer}, DLinear\cite{zeng2023transformers}, TimesNet\cite{wu2023timesnet} as benchmarks for comparison. It becomes evident that our model excels in predicting future trends, thus demonstrating its superior performance.

\begin{figure*}[htb!]
\centerline{\includegraphics[scale=0.17]{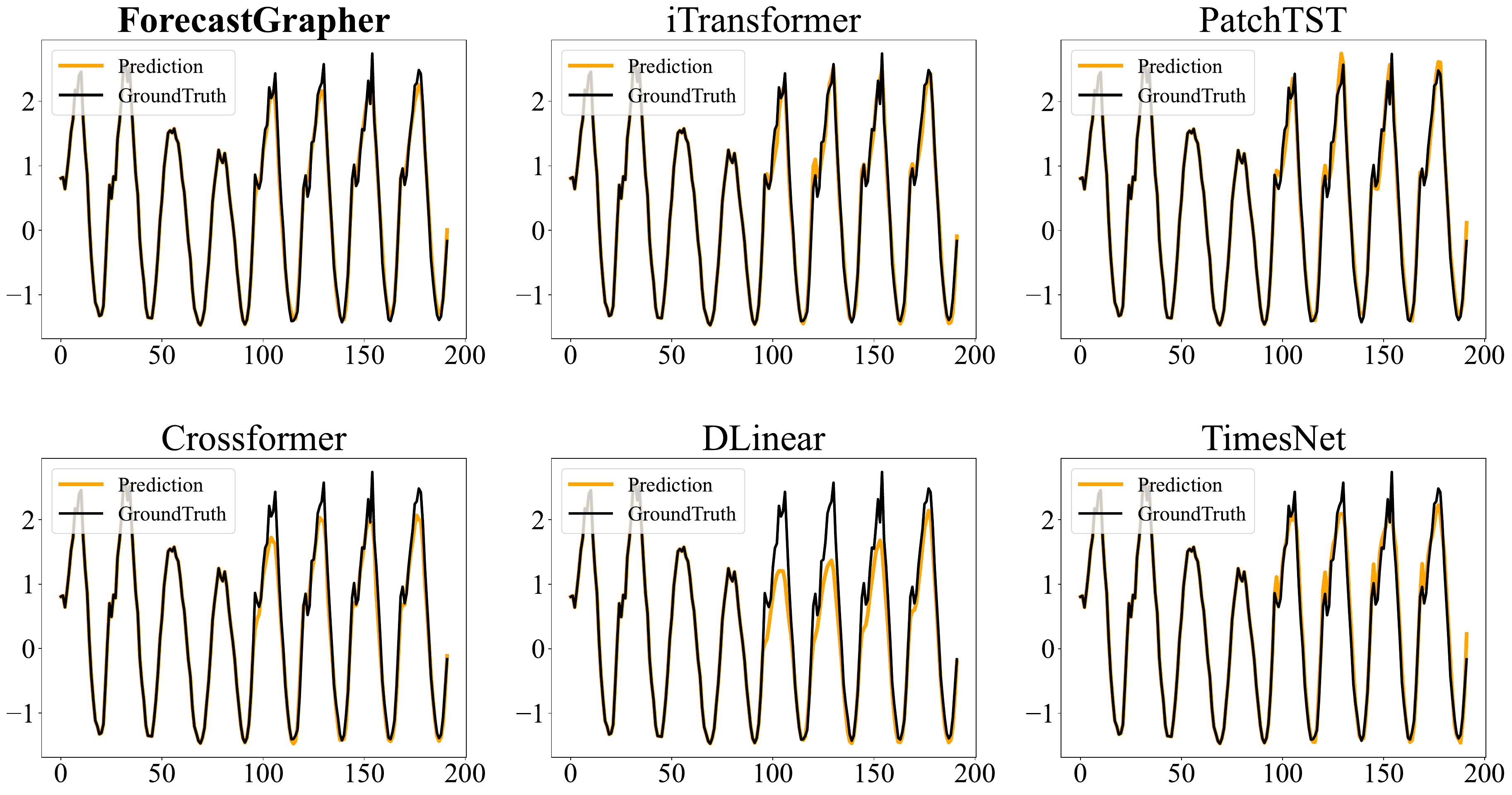}}
\caption{Visualization of input 96 and output 96 prediction results on the Traffic dataset.}
\label{fig:traffic_appendix}
\end{figure*}

\begin{figure*}[htb!]
\centerline{\includegraphics[scale=0.17]{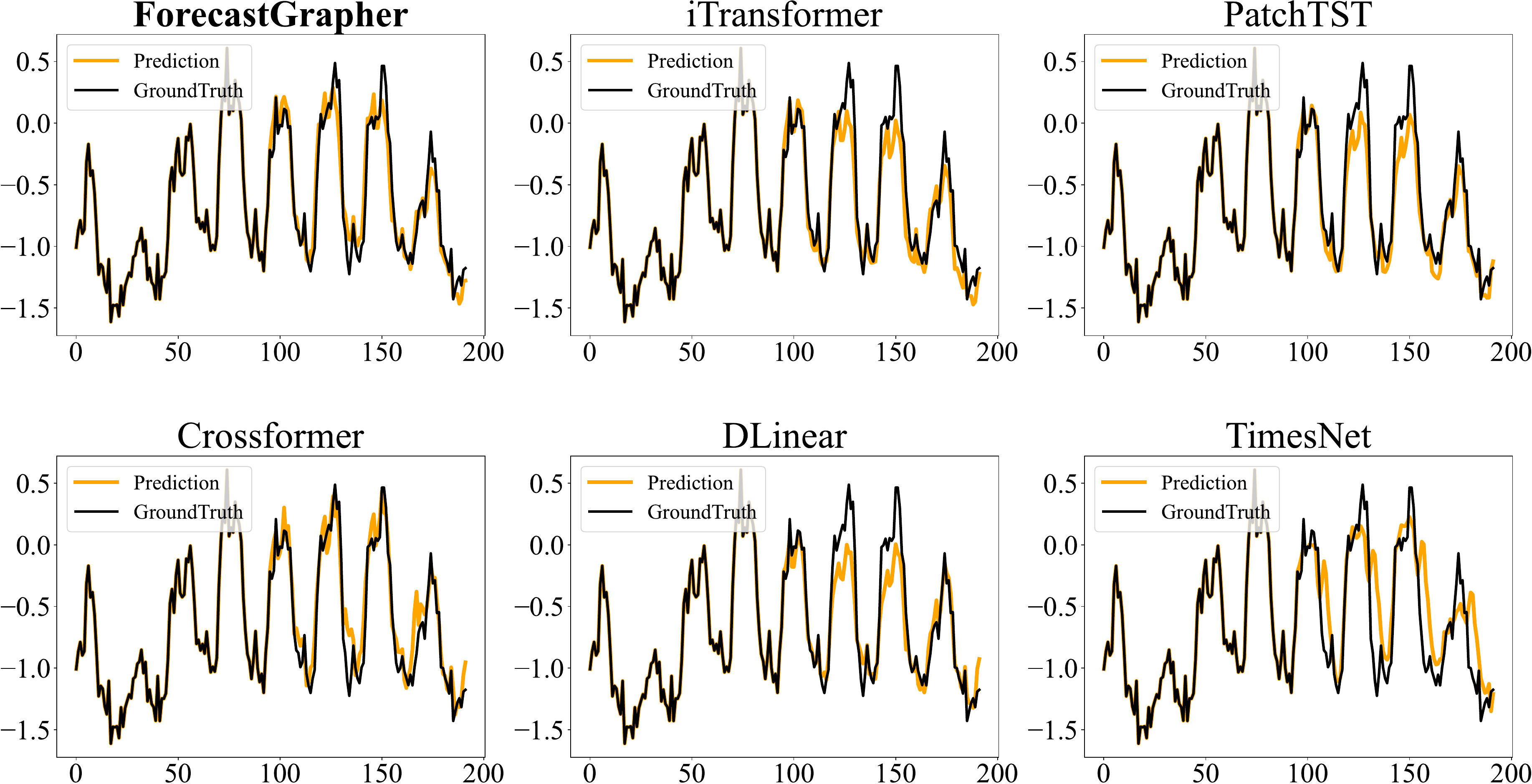}}
\caption{Visualization of input 96 and output 96 prediction results on the Electrcity dataset.}
\label{fig:ECL_appendix}
\end{figure*}

\begin{figure*}[htb!]
\centerline{\includegraphics[scale=0.17]{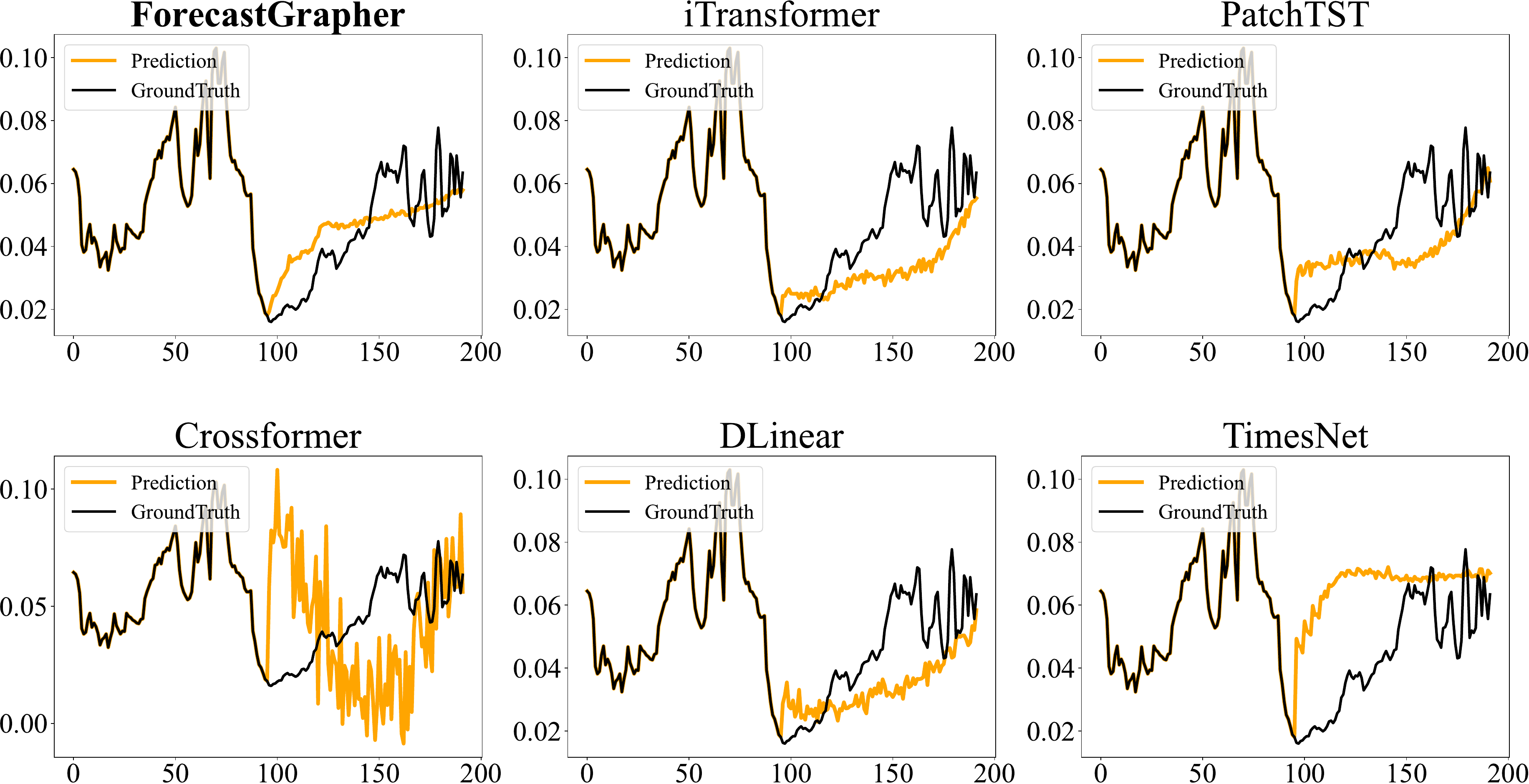}}
\caption{Visualization of input 96 and output 96 prediction results on the Weather dataset.}
\label{fig:weather_appendix}
\end{figure*}

\section{Full Forecasting Results}
\label{Full Forecasting Results}
We offer comprehensive multivariate prediction outcomes in this section. Tables \ref{tab:main_results} encompasses test across 12 benchmark datasets. The outcomes reveal that ForecastGrapher demonstrates outstanding performance across all datasets. Specifically, ForecastGrapher achieves the top spot in terms of MSE and MAE a total of 31 and 26 times, respectively. Table \ref{tab:main_results_gnn} contains comparison results with advanced GNNs and Naive method.

\begin{table*}[th]
\caption{Full multivariate time series prediction results, with input length 96, output lengths in \{12,24,48,96\} for PEMS, \{96,192,336,720\} for others. Use bold to indicate the best result, and underline to indicate the second. The benchmarks are reported from \cite{liu2023itransformer}}.
% \small
\tabcolsep=0.15cm
\renewcommand\arraystretch{1.1}
\centerline{
\resizebox{1\linewidth}{!}{
\begin{tabular}{ccccccccccccccccll}
\toprule
\multicolumn{2}{c}{Models}     & \multicolumn{2}{c}{\textbf{Ours}}                                             & \multicolumn{2}{c}{iTransformer}                                                  & \multicolumn{2}{c}{PatchTST}                                                   & \multicolumn{2}{c}{Crossformer}                                                   & \multicolumn{2}{c}{DLinear}                                                  & \multicolumn{2}{c}{RLinear}                                                & \multicolumn{2}{c}{TimesNet} &      \multicolumn{2}{c}{SCINet}\\ \midrule
\multicolumn{2}{c}{Metric}                                                    & MSE                                   & MAE                                   & MSE                                   & MAE                                   & MSE                                   & MAE                                   & MSE                                   & MAE                                   & MSE                                   & MAE                                & MSE                                   & MAE                                   & MSE           & MAE                         & MSE           &MAE                         \\ %\hline
\toprule
\multicolumn{1}{c|}{}                              & \multicolumn{1}{c|}{96}  & \textbf{0.317} & \textbf{0.357} & 0.334 & 0.368 & {\ul 0.329}    & {\ul 0.367}    & 0.404 & 0.426 & 0.345 & 0.372 & 0.355 & 0.376 & 0.338 & 0.375 & 0.418 & 0.438 \\

\multicolumn{1}{c|}{}                              & \multicolumn{1}{c|}{192} & \textbf{0.365} & \textbf{0.383} & 0.377 & 0.391 & {\ul 0.367}    & {\ul 0.385}    & 0.450 & 0.451 & 0.380 & 0.389 & 0.391 & 0.392 & 0.374 & 0.387 & 0.439 & 0.450 \\

\multicolumn{1}{c|}{}                              & \multicolumn{1}{c|}{336} & \textbf{0.393} & \textbf{0.405} & 0.426 & 0.420 & {\ul 0.399}    & {\ul 0.410}    & 0.532 & 0.515 & 0.413 & 0.413 & 0.424 & 0.415 & 0.410 & 0.411 & 0.490 & 0.485 \\

\multicolumn{1}{c|}{\multirow{-4}{*}{\rotatebox{90}{ETTm1}}}      & \multicolumn{1}{c|}{720} & {\ul 0.456}    & {\ul 0.442}    & 0.491 & 0.459 & \textbf{0.454} & \textbf{0.439} & 0.666 & 0.589 & 0.474 & 0.453 & 0.487 & 0.450 & 0.478 & 0.450 & 0.595 & 0.550 \\ 
\midrule

%2
\multicolumn{1}{c|}{}                              & \multicolumn{1}{c|}{96}  & 
{\ul 0.176}    & \textbf{0.259} & 0.180 & 0.264 & \textbf{0.175} & {\ul 0.259} & 0.287 & 0.366 & 0.193 & 0.292 & 0.182 & 0.265       & 0.187 & 0.267 & 0.286 & 0.377 \\

\multicolumn{1}{c|}{}                              & \multicolumn{1}{c|}{192} & \textbf{0.238} & \textbf{0.300} & 0.250 & 0.309 & {\ul 0.241}    & {\ul 0.302} & 0.414 & 0.492 & 0.284 & 0.362 & 0.246 & 0.304       & 0.249 & 0.309 & 0.399 & 0.445 \\

\multicolumn{1}{c|}{}                              & \multicolumn{1}{c|}{336} & \textbf{0.296} & \textbf{0.338} & 0.311 & 0.348 & {\ul 0.305}    & 0.343       & 0.597 & 0.542 & 0.369 & 0.427 & 0.307 & {\ul 0.342} & 0.321 & 0.351 & 0.637 & 0.591 \\

\multicolumn{1}{c|}{\multirow{-4}{*}{\rotatebox{90}{ETTm2}}}     & \multicolumn{1}{c|}{720} & \textbf{0.395} & \textbf{0.396} & 0.412 & 0.407 & {\ul 0.402}    & 0.400       & 1.730 & 1.042 & 0.554 & 0.522 & 0.407 & {\ul 0.398} & 0.408 & 0.403 & 0.960 & 0.735\\ 
\midrule

%3
\multicolumn{1}{c|}{}                              & \multicolumn{1}{c|}{96}  & \textbf{0.373} & {\ul 0.397} & 0.386 & 0.405 & 0.414 & 0.419 & 0.423 & 0.448 & 0.386 & 0.400 & 0.386       & \textbf{0.395} & {\ul 0.384} & 0.402 & 0.654 & 0.599 \\

\multicolumn{1}{c|}{}                              & \multicolumn{1}{c|}{192} & \textbf{0.424} & {\ul 0.427} & 0.441 & 0.436 & 0.460 & 0.445 & 0.471 & 0.474 & 0.437 & 0.432 & 0.437       & \textbf{0.424} & {\ul 0.436} & 0.429 & 0.719 & 0.631 \\

\multicolumn{1}{c|}{}                              & \multicolumn{1}{c|}{336} & \textbf{0.472} & {\ul 0.448} & 0.487 & 0.458 & 0.501 & 0.466 & 0.570 & 0.546 & 0.481 & 0.459 & {\ul 0.479} & \textbf{0.446} & 0.491       & 0.469 & 0.778 & 0.659 \\

\multicolumn{1}{c|}{\multirow{-4}{*}{\rotatebox{90}{ETTh1}}}       & \multicolumn{1}{c|}{720} & \textbf{0.479} & {\ul 0.475} & 0.503 & 0.491 & 0.500 & 0.488 & 0.653 & 0.621 & 0.519 & 0.516 & {\ul 0.481} & \textbf{0.470} & 0.521       & 0.500 & 0.836 & 0.699 \\ 
\midrule

%4
\multicolumn{1}{c|}{}                              & \multicolumn{1}{c|}{96}  & {\ul 0.294}    & {\ul 0.345} & 0.297 & 0.349 & 0.302 & 0.348 & 0.745 & 0.584 & 0.333 & 0.387 & \textbf{0.288} & \textbf{0.338} & 0.340 & 0.374 & 0.707 & 0.621 \\

\multicolumn{1}{c|}{}                              & \multicolumn{1}{c|}{192} & \textbf{0.367} & {\ul 0.396} & 0.380 & 0.400 & 0.388 & 0.400 & 0.877 & 0.656 & 0.477 & 0.476 & {\ul 0.374}    & \textbf{0.390} & 0.402 & 0.414 & 0.860 & 0.689 \\

\multicolumn{1}{c|}{}                              & \multicolumn{1}{c|}{336} & \textbf{0.407} & {\ul 0.428} & 0.428 & 0.432 & 0.426 & 0.433 & 1.043 & 0.731 & 0.594 & 0.541 & {\ul 0.415}    & \textbf{0.426} & 0.452 & 0.452 & 1.000 & 0.744 \\

\multicolumn{1}{c|}{\multirow{-4}{*}{\rotatebox{90}{ETTh2}}}       & \multicolumn{1}{c|}{720} & \textbf{0.420} & {\ul 0.441} & 0.427 & 0.445 & 0.431 & 0.446 & 1.104 & 0.763 & 0.831 & 0.657 & {\ul 0.420}    & \textbf{0.440} & 0.462 & 0.468 & 1.249 & 0.838 \\ 
\midrule

%5
\multicolumn{1}{c|}{}                              & \multicolumn{1}{c|}{96}  & \textbf{0.138} & \textbf{0.235} & {\ul 0.148} & {\ul 0.240} & 0.181& 0.270& 0.219 & 0.314 & 0.197 & 0.282 & 0.201 & 0.281 & 0.168       & 0.272 & 0.247 & 0.345 \\

\multicolumn{1}{c|}{}                              & \multicolumn{1}{c|}{192} & \textbf{0.154} & \textbf{0.249} & {\ul 0.162} & {\ul 0.253} & 0.188& 0.274& 0.231 & 0.322 & 0.196 & 0.285 & 0.201 & 0.283 & 0.184       & 0.289 & 0.257 & 0.355 \\

\multicolumn{1}{c|}{}                              & \multicolumn{1}{c|}{336} & \textbf{0.169} & \textbf{0.264} & {\ul 0.178} & {\ul 0.269} & 0.204& 0.293& 0.246 & 0.337 & 0.209 & 0.301 & 0.215 & 0.298 & 0.198       & 0.300 & 0.269 & 0.369 \\

\multicolumn{1}{c|}{\multirow{-4}{*}{\rotatebox{90}{Electricity}}}       & \multicolumn{1}{c|}{720} & \textbf{0.199} & \textbf{0.294} & 0.225       & {\ul 0.317} & 0.246& 0.324& 0.280 & 0.363 & 0.245 & 0.333 & 0.257 & 0.331 & {\ul 0.220} & 0.320 & 0.299 & 0.390 \\ 
\midrule

%6
\multicolumn{1}{c|}{}                              & \multicolumn{1}{c|}{96}  & \textbf{0.086} & {\ul 0.206} & {\ul 0.086} & 0.206          & 0.088          & \textbf{0.205} & 0.256 & 0.367 & 0.088          & 0.218       & 0.093 & 0.217 & 0.107 & 0.234 & 0.267 & 0.396 \\

\multicolumn{1}{c|}{}                              & \multicolumn{1}{c|}{192} & 0.181          & 0.303       & 0.177       & {\ul 0.299}    & \textbf{0.176} & \textbf{0.299} & 0.470 & 0.509 & {\ul 0.176}    & 0.315       & 0.184 & 0.307 & 0.226 & 0.344 & 0.351 & 0.459 \\

\multicolumn{1}{c|}{}                              & \multicolumn{1}{c|}{336} & 0.334          & 0.418       & 0.331       & {\ul 0.417}    & \textbf{0.301} & \textbf{0.397} & 1.268 & 0.883 & {\ul 0.313}    & 0.427       & 0.351 & 0.432 & 0.367 & 0.448 & 1.324 & 0.853 \\

\multicolumn{1}{c|}{\multirow{-4}{*}{\rotatebox{90}{Exchange}}}       & \multicolumn{1}{c|}{720} & 0.869          & 0.702       & {\ul 0.847} & \textbf{0.691} & 0.901          & 0.714          & 1.767 & 1.068 & \textbf{0.839} & {\ul 0.695} & 0.886 & 0.714 & 0.964 & 0.746 & 1.058 & 0.797\\ 
\midrule

%7
\multicolumn{1}{c|}{}                              & \multicolumn{1}{c|}{96}  & {\ul 0.423} & {\ul 0.278} & \textbf{0.395} & \textbf{0.268} & 0.462& 0.295& 0.522 & 0.290 & 0.650 & 0.396 & 0.649 & 0.389 & 0.593 & 0.321 & 0.788 & 0.499 \\

\multicolumn{1}{c|}{}                              & \multicolumn{1}{c|}{192} & {\ul 0.445} & {\ul 0.285} & \textbf{0.417} & \textbf{0.276} & 0.466& 0.296& 0.530 & 0.293 & 0.598 & 0.370 & 0.601 & 0.366 & 0.617 & 0.336 & 0.789 & 0.505 \\

\multicolumn{1}{c|}{}                              & \multicolumn{1}{c|}{336} & {\ul 0.460} & {\ul 0.292} & \textbf{0.433} & \textbf{0.283} & 0.482 & 0.304 & 0.558 & 0.305 & 0.605 & 0.373 & 0.609 & 0.369 & 0.629 & 0.336 & 0.797 & 0.508 \\

\multicolumn{1}{c|}{\multirow{-4}{*}{\rotatebox{90}{Traffic}}} & \multicolumn{1}{c|}{720} & {\ul 0.503} & {\ul 0.312} & \textbf{0.467} & \textbf{0.302} & 0.514 & 0.322 & 0.589 & 0.328 & 0.645 & 0.394 & 0.647 & 0.387 & 0.640 & 0.350 & 0.841 & 0.523 \\ \midrule

%8
\multicolumn{1}{c|}{}                              & \multicolumn{1}{c|}{96}  & {\ul 0.161}    & \textbf{0.206} & 0.174 & {\ul 0.214} & 0.177 & 0.218       & \textbf{0.158} & 0.230 & 0.196          & 0.255 & 0.192 & 0.232 & 0.172 & 0.220 & 0.222 & 0.306 \\

\multicolumn{1}{c|}{}                              & \multicolumn{1}{c|}{192} & {\ul 0.209}    & \textbf{0.251} & 0.221 & {\ul 0.254} & 0.225 & 0.259       & \textbf{0.206} & 0.277 & 0.237          & 0.296 & 0.240 & 0.271 & 0.219 & 0.261 & 0.261 & 0.340 \\

\multicolumn{1}{c|}{}                              & \multicolumn{1}{c|}{336} & \textbf{0.268} & \textbf{0.295} & 0.278 & {\ul 0.296} & 0.278 & 0.297       & {\ul 0.272}    & 0.335 & 0.283          & 0.335 & 0.292 & 0.307 & 0.280 & 0.306 & 0.309 & 0.378 \\

\multicolumn{1}{c|}{\multirow{-4}{*}{\rotatebox{90}{Weather}}}    & \multicolumn{1}{c|}{720} & {\ul 0.348}    & \textbf{0.345} & 0.358 & 0.349       & 0.354 & {\ul 0.348} & 0.398          & 0.418 & \textbf{0.345} & 0.381 & 0.364 & 0.353 & 0.365 & 0.359 & 0.377 & 0.427 \\ 
\midrule

%9
\multicolumn{1}{c|}{}                              & \multicolumn{1}{c|}{12
}  & \textbf{0.065} & \textbf{0.168} & 0.071       & 0.174       & 0.099 & 0.216 & 0.090 & 0.203 & 0.122 & 0.243 & 0.126 & 0.236 & 0.085 & 0.192 & {\ul 0.066} & {\ul 0.172} \\

\multicolumn{1}{c|}{}                              & \multicolumn{1}{c|}{24
} & \textbf{0.081} & \textbf{0.186} & 0.093       & 0.201       & 0.142 & 0.259 & 0.121 & 0.240 & 0.201 & 0.317 & 0.246 & 0.334 & 0.118 & 0.223 & {\ul 0.085} & {\ul 0.198} \\

\multicolumn{1}{c|}{}                              & \multicolumn{1}{c|}{48
} & \textbf{0.111} & \textbf{0.220} & {\ul 0.125} & {\ul 0.236} & 0.211 & 0.319 & 0.202 & 0.317 & 0.333 & 0.425 & 0.551 & 0.529 & 0.155 & 0.260 & 0.127       & 0.238       \\

\multicolumn{1}{c|}{\multirow{-4}{*}{\rotatebox{90}{PEMS03}}}    & \multicolumn{1}{c|}{96
} & \textbf{0.134} & \textbf{0.244} & {\ul 0.164} & {\ul 0.275} & 0.269 & 0.370 & 0.262 & 0.367 & 0.457 & 0.515 & 1.057 & 0.787 & 0.228 & 0.317 & 0.178       & 0.287    \\ \midrule

%10
\multicolumn{1}{c|}{}                              & \multicolumn{1}{c|}{12
}  & {\ul 0.075}    & {\ul 0.181}    & 0.078 & 0.183 & 0.105 & 0.224 & 0.098 & 0.218 & 0.148 & 0.272 & 0.138 & 0.252 & 0.087 & 0.195 & \textbf{0.073} & \textbf{0.177} \\

\multicolumn{1}{c|}{}                              & \multicolumn{1}{c|}{24} & {\ul 0.085}    & {\ul 0.194}    & 0.095 & 0.205 & 0.153 & 0.275 & 0.131 & 0.256 & 0.224 & 0.340 & 0.258 & 0.348 & 0.103 & 0.215 & \textbf{0.084} & \textbf{0.193} \\

\multicolumn{1}{c|}{}                              & \multicolumn{1}{c|}{48
} & {\ul 0.099}    & {\ul 0.213}    & 0.120 & 0.233 & 0.229 & 0.339 & 0.205 & 0.326 & 0.355 & 0.437 & 0.572 & 0.544 & 0.136 & 0.250 & \textbf{0.099} & \textbf{0.211} \\
 
\multicolumn{1}{c|}{\multirow{-4}{*}{\rotatebox{90}{PEMS04}}}    & \multicolumn{1}{c|}{96
} & \textbf{0.112} & \textbf{0.227} & 0.150 & 0.262 & 0.291 & 0.389 & 0.402 & 0.457 & 0.452 & 0.504 & 1.137 & 0.820 & 0.190 & 0.303 & {\ul 0.114}    & {\ul 0.227}   
\\ \midrule

%11
\multicolumn{1}{c|}{}                              & \multicolumn{1}{c|}{12
}  & \textbf{0.058} & \textbf{0.152} & {\ul 0.067} & {\ul 0.165} & 0.095 & 0.207 & 0.094 & 0.200 & 0.115 & 0.242 & 0.118 & 0.235 & 0.082 & 0.181 & 0.068 & 0.171       \\

\multicolumn{1}{c|}{}                              & \multicolumn{1}{c|}{24
} & \textbf{0.069} & \textbf{0.163} & {\ul 0.088} & {\ul 0.190} & 0.150 & 0.262 & 0.139 & 0.247 & 0.210 & 0.329 & 0.242 & 0.341 & 0.101 & 0.204 & 0.119 & 0.225       \\

\multicolumn{1}{c|}{}                              & \multicolumn{1}{c|}{48
} & \textbf{0.085} & \textbf{0.179} & {\ul 0.110} & {\ul 0.215} & 0.253 & 0.340 & 0.311 & 0.369 & 0.398 & 0.458 & 0.562 & 0.541 & 0.134 & 0.238 & 0.149 & 0.237       \\

\multicolumn{1}{c|}{\multirow{-4}{*}{\rotatebox{90}{PEMS07}}}    & \multicolumn{1}{c|}{96} & \textbf{0.103} & \textbf{0.194} & {\ul 0.139} & 0.245       & 0.346 & 0.404 & 0.396 & 0.442 & 0.594 & 0.553 & 1.096 & 0.795 & 0.181 & 0.279 & 0.141 & {\ul 0.234}
\\ \midrule

%12
\multicolumn{1}{c|}{}                              & \multicolumn{1}{c|}{12
}  & {\ul 0.081}    & {\ul 0.184}    & \textbf{0.079} & \textbf{0.182} & 0.168 & 0.232 & 0.165 & 0.214 & 0.154 & 0.276 & 0.133 & 0.247 & 0.112 & 0.212 & 0.087 & 0.184 \\

\multicolumn{1}{c|}{}                              & \multicolumn{1}{c|}{24} & \textbf{0.115} & {\ul 0.220}    & {\ul 0.115}    & \textbf{0.219} & 0.224 & 0.281 & 0.215 & 0.260 & 0.248 & 0.353 & 0.249 & 0.343 & 0.141 & 0.238 & 0.122 & 0.221 \\

\multicolumn{1}{c|}{}                              & \multicolumn{1}{c|}{48} & \textbf{0.169} & \textbf{0.211} & {\ul 0.186}    & {\ul 0.235}    & 0.321 & 0.354 & 0.315 & 0.355 & 0.440 & 0.470 & 0.569 & 0.544 & 0.198 & 0.283 & 0.189 & 0.270 \\

\multicolumn{1}{c|}{\multirow{-4}{*}{\rotatebox{90}{PEMS08}}}    & \multicolumn{1}{c|}{96
} & \textbf{0.197} & \textbf{0.234} & {\ul 0.221}    & {\ul 0.267}    & 0.408 & 0.417 & 0.377 & 0.397 & 0.674 & 0.565 & 1.166 & 0.814 & 0.320 & 0.351 & 0.236 & 0.300
\\ \midrule

\multicolumn{2}{c|}{$1^{st}$ Count}   &\textbf{31} &\multicolumn{1}{c|}{\textbf{26}} &   {\ul 5}& \multicolumn{1}{c|}{7}&  4& \multicolumn{1}{c|}{4}  &2  &\multicolumn{1}{c|}{0}  &2  &\multicolumn{1}{c|}{0}  &1   &\multicolumn{1}{c|}{{\ul 8}} &0 &\multicolumn{1}{c|}{0}  &\multicolumn{1}{c}{3}  &\multicolumn{1}{c}{3}
\\
\bottomrule

\end{tabular}
}
}

\label{tab:main_results}
\end{table*}

\begin{table*}[th]
\caption{Full comparison with GNNs and Naive method for multivariate time series prediction, with input length 96, output lengths in \{12,24,48,96\} for PEMS, \{96,192,336,720\} for others. Use bold to indicate the best result and the symbol '-' indicates exceeding memory.}
% \small
\tabcolsep=0.15cm
\renewcommand\arraystretch{1.2}
\centerline{
\resizebox{1\linewidth}{!}{
\begin{tabular}{cccccccccccccccc}
\toprule
\multicolumn{2}{c}{Dataset}     & \multicolumn{2}{c}{Electricity}                                             & \multicolumn{2}{c}{Traffic}                                                  & \multicolumn{2}{c}{Weather}                                                   & \multicolumn{2}{c}{PEMS03}                                                   & \multicolumn{2}{c}{PEMS04}                                                  & \multicolumn{2}{c}{PEMS07}                                                & \multicolumn{2}{c}{PEMS08} \\ \midrule
\multicolumn{2}{c}{Metric}                                                    & MSE                                   & MAE                                   & MSE                                   & MAE                                   & MSE                                   & MAE                                   & MSE                                   & MAE                                   & MSE                                   & MAE                                & MSE                                   & MAE                                   & MSE           & MAE                         \\ %\hline
\toprule
\multicolumn{1}{c|}{}                              & \multicolumn{1}{c|}{96(12)}  
& \textbf{0.138} & \textbf{0.235}
& \textbf{0.423} & \textbf{0.278}
& \textbf{0.161} & \textbf{0.206}
& \textbf{0.065} & \textbf{0.168}
& \textbf{0.075} & \textbf{0.181}
& \textbf{0.058} & \textbf{0.152}
& \textbf{0.081} & \textbf{0.184}
\\

\multicolumn{1}{c|}{}                              & \multicolumn{1}{c|}{192(24)} 
& \textbf{0.154} & \textbf{0.249} 
& \textbf{0.445} & \textbf{0.285} 
& \textbf{0.209} & \textbf{0.251} 
& \textbf{0.081} & \textbf{0.186} 
& \textbf{0.085} & \textbf{0.194} 
& \textbf{0.069} & \textbf{0.163} 
& \textbf{0.115} & \textbf{0.220} 
\\

\multicolumn{1}{c|}{}                              & \multicolumn{1}{c|}{336(48)} 
& \textbf{0.169} & \textbf{0.264} 
& \textbf{0.460} & \textbf{0.292} 
&  0.268 & \textbf{0.295}
& \textbf{0.111} & \textbf{0.220} 
& \textbf{0.099} & \textbf{0.213} 
& \textbf{0.085} & \textbf{0.179} 
& \textbf{0.169} & \textbf{0.211} 
\\

\multicolumn{1}{c|}{\multirow{-4}{*}{\rotatebox{90}{\textbf{Ours}}}}      & \multicolumn{1}{c|}{720(96)} 
& \textbf{0.199} & \textbf{0.294} 
& \textbf{0.503} & \textbf{0.312} 
&  0.348 & \textbf{0.345} 
& \textbf{0.134} & \textbf{0.244} 
& \textbf{0.112} & \textbf{0.227} 
& \textbf{0.103} & \textbf{0.194} 
& \textbf{0.197} & \textbf{0.234} 
\\ 
\midrule

%2
\multicolumn{1}{c|}{}                              & \multicolumn{1}{c|}{96(12)
}  & 
0.211 & 0.307 
& 0.538 & 0.335 
& 0.177 & 0.240 
& 0.087 & 0.202 
& 0.112 & 0.231 
& 0.073 & 0.182 
& 0.143 & 0.263 
\\

\multicolumn{1}{c|}{}                              & \multicolumn{1}{c|}{192(24)
} & 0.214 & 0.312 
& 0.536 & 0.334 
& 0.218 & 0.279 
& 0.120 & 0.240 
& 0.153 & 0.272 
& 0.100 & 0.215 
& 0.210 & 0.320 
\\

\multicolumn{1}{c|}{}                              & \multicolumn{1}{c|}{336(48)
} & 0.227 & 0.325 
& 0.556 & 0.340 
& \textbf{0.265} & 0.318 
& 0.177 & 0.294 
& 0.209 & 0.321 
& 0.140 & 0.258 
& 0.216 & 0.311 
\\

\multicolumn{1}{c|}{\multirow{-4}{*}{\rotatebox{90}{FourierGNN}}}     & \multicolumn{1}{c|}{720(96)} & 0.260 & 0.354 
& 0.597 & 0.358 
& \textbf{0.336} & 0.370 
& 0.218 & 0.333 
& 0.247 & 0.354 
& 0.177 & 0.292 
& 0.294 & 0.356 
\\ 
\midrule

%3
\multicolumn{1}{c|}{}                              & \multicolumn{1}{c|}{96(12)
}  & 0.165 & 0.267 
& 0.576 & 0.339 
& 0.181 & 0.250 
& 0.119 & 0.244 
& 0.144 & 0.276 
& 0.120 & 0.242 
& 0.246 & 0.319 
\\

\multicolumn{1}{c|}{}                              & \multicolumn{1}{c|}{192(24)
} & 0.180 & 0.283 
& 0.593 & 0.345 
& 0.226 & 0.289 
& 0.179 & 0.305 
& 0.188 & 0.317 
& 0.168 & 0.282 
& 0.281 & 0.337 
\\

\multicolumn{1}{c|}{}                              & \multicolumn{1}{c|}{336(48)
} & 0.200 & 0.306 
& 0.624 & 0.366 
& 0.287 & 0.338 
& 0.191 & 0.303 
& 0.234 & 0.342 
& 0.184 & 0.285 
& 0.305 & 0.356 
\\

\multicolumn{1}{c|}{\multirow{-4}{*}{\rotatebox{90}{StemGNN}}}       & \multicolumn{1}{c|}{720(96)} & 0.243 & 0.345 
& 0.655 & 0.373 
& 0.379 & 0.406 
& 0.258 & 0.355 
& 0.303 & 0.396 
& 0.265 & 0.346 
& 0.380 & 0.393 
\\ 
\midrule

%12
\multicolumn{1}{c|}{}                              & \multicolumn{1}{c|}{96(12)
}  & 0.321 & 0.326 
& 1.222 & 0.499 
& 0.290 & 0.284 
& 0.541 & 0.542 
& 0.573 & 0.566 
& 0.581 & 0.559 
& 0.577 & 0.575 
\\

\multicolumn{1}{c|}{}                              & \multicolumn{1}{c|}{192(24)
} & 0.304 & 0.323 
& 1.095 & 0.458 
& 0.331 & 0.311 
& 0.541 & 0.543 
& 0.573 & 0.567 
& 0.582 & 0.560 
& 0.577 & 0.573 
\\

\multicolumn{1}{c|}{}                              & \multicolumn{1}{c|}{336(48)
} & 0.326 & 0.342 
& 1.151 & 0.475 
& 0.392 & 0.350 
& 0.914 & 0.722 
& 0.977 & 0.754 
& 1.003 & 0.746 
& 1.019 & 0.770 
\\

\multicolumn{1}{c|}{\multirow{-4}{*}{\rotatebox{90}{Naive}}}    & \multicolumn{1}{c|}{720(96)} & 0.366 & 0.373 
&- &-& 0.472 & 0.399 
& 1.610 & 1.004 
& 1.742 & 1.053 
& 1.698 & 1.016 
& 1.816 & 1.061 
\\
\bottomrule
\end{tabular}
}
}

\label{tab:main_results_gnn}
\end{table*}

\section{Broader Impact}
\label{Broader Impact}

The proposed model, ForecastGrapher holds some potential impacts in multivariate time series forecasting and machine learning domains. It introduces a GNN framework specifically tailored for multivariate time series forecasting, enhancing the the ability to capture and express complex time series correlations. It achieves state-of-the-art performance on real-world datasets, making it more promising for practical applications like weather and electricity forecasting. Moreover, it serves as a valuable, unified modeling framework for time series correlations, offering a new path for future research. However, predictions about stocks, which involve significant uncertainty, may yield incorrect outcomes, harming the profit of investors.

\end{document}